\documentclass{article} %
\usepackage{iclr2027_conference,times}

\usepackage[utf8]{inputenc} %
\usepackage[T1]{fontenc}    %
\usepackage{hyperref}       %
\usepackage{url}            %
\usepackage{booktabs}       %
\usepackage{amsfonts}       %
\usepackage{nicefrac}       %
\usepackage{microtype}      %
\usepackage{xcolor}         %
\usepackage{algorithmic}
\usepackage[ruled,vlined,linesnumbered]{algorithm2e}
\usepackage{amsmath}
\usepackage{amssymb}
\usepackage{mathtools}
\usepackage{amsthm}
\usepackage{bm}
\usepackage{multirow} %
\usepackage{wrapfig}

\usepackage{tikz}
\usetikzlibrary{shapes,arrows,positioning,fit,calc,decorations.pathreplacing,backgrounds,patterns,pgfplots.groupplots}

\usepackage{pgfplots}
\pgfplotsset{compat=1.17}
\usepackage{enumitem}
\setlist{nosep,leftmargin=*}
\usepackage{comment}
\usepackage{xspace}
\usepackage{nicefrac}
\usepackage{arydshln}

\usepackage{mdframed}
\definecolor{theoremcolor}{rgb}{0.94, 0.94, 0.94}
\definecolor{examplecolor}{rgb}{1, 1, 1.0}
\mdfsetup{
    backgroundcolor=theoremcolor,
    linewidth=0pt,
}
\newmdtheoremenv[linewidth=0pt,innerleftmargin=4pt,innerrightmargin=4pt]{definition}{Definition}
\newmdtheoremenv[linewidth=0pt,innerleftmargin=4pt,innerrightmargin=4pt]{proposition}{Proposition}
\newmdtheoremenv[linewidth=0pt,innerleftmargin=0pt,innerrightmargin=0pt,backgroundcolor=examplecolor]{example}{Example}
\newmdtheoremenv[linewidth=0pt,innerleftmargin=4pt,innerrightmargin=4pt]{corollary}{Corollary}
\newmdtheoremenv[linewidth=0pt,innerleftmargin=4pt,innerrightmargin=4pt]{theorem}{Theorem}
\newmdtheoremenv[linewidth=0pt,innerleftmargin=4pt,innerrightmargin=4pt]{lemma}{Lemma}
\newmdtheoremenv[linewidth=0pt,innerleftmargin=4pt,innerrightmargin=4pt]{remark}{Remark}

\definecolor{questblue}{RGB}{66,133,244}
\definecolor{entmaxgreen}{RGB}{52,168,83}
\definecolor{softmaxred}{RGB}{234,67,53}
\definecolor{flashorange}{RGB}{251,188,5}
\definecolor{dropgray}{RGB}{180,180,180}
\definecolor{h2opurple}{RGB}{153,102,204}
\definecolor{keepgreen}{RGB}{200,230,200}
\definecolor{dropred}{RGB}{255,220,220}

\newcommand{\R}{\mathbb{R}}
\newcommand{\E}{\mathbb{E}}

\newcommand{\softmax}{\text{softmax}}
\newcommand{\entmax}{\alpha\text{-entmax}}
\newcommand{\supp}{\mathrm{supp}}

\newcommand{\simark}{\mathcal{S}}  %
\newcommand{\Ikeep}{\mathcal{I}_{\text{keep}}}
\newcommand{\Idrop}{\mathcal{I}_{\text{drop}}}

\newcommand{\methodname}{EntmaxKV\xspace}

\title{\methodname: Support-Aware Decoding for Entmax Attention}

\author{
\textbf{Gonçalo Duarte\textsuperscript{1}}~,~
\textbf{Miguel Couceiro\textsuperscript{1,2,3}}~,~
\textbf{Marcos V. Treviso\textsuperscript{1,2,4}}
\\
\textsuperscript{1}Instituto Superior T\'ecnico, Universidade de Lisboa. 
\textsuperscript{2}ELLIS Unit Lisbon.
\\
\textsuperscript{3}INESC-ID. 
\textsuperscript{4}Instituto de Telecomunica\c{c}\~oes.
}

\iclrfinalcopy %
\begin{document}

\maketitle

\begin{abstract}
Long-context decoding is increasingly limited by KV-cache memory traffic as each generated token attends over a cache whose size grows linearly with context length. 
Existing sparse decoding methods reduce this cost by selecting subsets of tokens or pages, but are designed for softmax attention, whose dense tails make any truncation discard nonzero probability mass. 
In contrast, $\alpha$-entmax produces exact zeros, turning sparse decoding from dense-tail approximation into support recovery: if the selected candidates contain the entmax support, sparse decoding remains exact.
While recent entmax kernels enable efficient training, they do not address the autoregressive decoding bottleneck, where dense inference still streams the full KV cache before sparsity is known. 
In this work, we introduce \methodname, an entmax-native sparse decoding framework that exploits sparsity before KV pages are loaded. \methodname combines query-aware page scoring, support-aware candidate selection, and sparse entmax attention. We analyze truncation error through the dropped probability mass $\delta$, showing that output error is controlled by $\delta$ and vanishes when the entmax support is recovered. We further introduce a Gaussian-aware entmax selector that estimates the entmax threshold from lightweight page statistics, adapting the selected budget to the score distribution.
Empirically, \methodname drops less probability mass, retains more support tokens, and achieves lower output error than softmax-based sparse decoding at matched KV budgets. On long-context and language modeling benchmarks, it closely matches full-cache entmax while using a small fraction of the KV cache, achieving up to $2.88\times$ (softmax) and $5.5\times$ (entmax) speedup over full attention baselines as well as over $2\times$ speedup over softmax and entmax top-$k$ at 1M context length.\footnote{Our code is available at \url{https://github.com/deep-spin/entmaxkv}.}
\end{abstract}

\section{Introduction}
\label{sec:intro}

Recent large language models (LLMs) support increasingly long context windows, with some systems extending to hundreds of thousands or even millions of tokens \citep{meta2025llama,team2025kimi}. This   reasoning over long documents, code repositories, conversation histories, and retrieval-augmented inputs. However, long contexts come at substantial computational cost, especially during autoregressive decoding: at each generation step, the model must attend over the key-value (KV) cache containing all previous tokens. As a result, long-context decoding is often memory-bound, and KV-cache reads can account for over half of the per-token latency \citep{tang2024quest}.

A natural way to reduce this cost is to avoid reading the entire cache. Existing sparse decoding methods select only a subset of tokens, blocks, or pages for attention. Some approaches permanently discard tokens from the KV cache using positional, learned, or attention-based heuristics \citep{zhang2023ho,xiao2024efficient,oren-etal-2024-transformers,li2024snapkv}. Others keep the full cache available but select a query-dependent subset at each decoding step, using page-level bounds, clustering, or adaptive thresholds to identify the tokens that matter for the current query \citep{tang2024quest,lin2025twilight,singhania2024loki,nawrot2024dynamic,zhu2025tactic}. These methods exploit the empirical observation that attention is often concentrated on a small subset of the context. However, most sparse decoding methods are designed around softmax attention. Since softmax assigns strictly positive probability to every token, any truncation of the KV cache necessarily discards nonzero probability mass.

In this work, we revisit sparse decoding through the lens of $\alpha$-entmax attention \citep{peters-etal-2019-sparse,blondel-entmax}. Unlike softmax, $\alpha$-entmax can produce probability distributions with exact zeros. 
Therefore, rather than approximating a dense distribution with a truncated one as done by softmax-based approaches, an entmax decoder only needs to recover the support. 
As shown in Figure~\ref{fig:concept}, if the selected KV pages contain all tokens in the entmax support, sparse decoding is exactly equal to full-cache entmax attention while avoiding the memory traffic associated with irrelevant pages.\looseness=-1

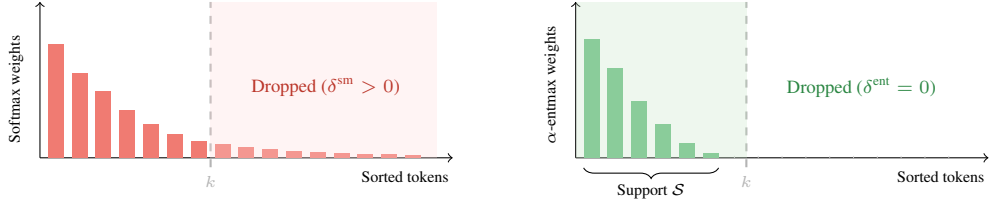
\begin{figure}[t]
\centering
\begin{tikzpicture}[xscale=1.05, yscale=1.25]
    
    \begin{scope}[xshift=0cm]
        \node[font=\tiny, anchor=east, rotate=90] at (-0.3, 1.9) {Softmax weights};
        \draw[->] (0, 0.5) -- (5.2, 0.5);
        \draw[->] (0, 0.5) -- (0, 2.15);
        
        \foreach \x/\h in {0.2/1.2, 0.5/0.9, 0.8/0.7, 1.1/0.5, 1.4/0.35, 1.7/0.25, 2.0/0.18, 2.3/0.14, 2.6/0.11, 2.9/0.09, 3.2/0.07, 3.5/0.06, 3.8/0.05, 4.1/0.04, 4.4/0.035, 4.7/0.03} {
            \fill[softmaxred!70] (\x-0.1, 0.5) rectangle (\x+0.1, 0.5+\h);
        }
        
        \draw[dashed, thick, dropgray] (2.15, 0.4) -- (2.15, 2.15);
        \node[font=\tiny, text=dropgray] at (2.15, 0.25) {$k$};
        \fill[dropred, opacity=0.3] (2.15, 0.5) rectangle (5, 2.15);
        \node[font=\scriptsize, text=softmaxred!80!black] at (3.6, 1.25) {Dropped ($\delta^{\text{sm}} > 0$)};
        
        \node[font=\tiny] at (4.6, 0.3) {Sorted tokens};
    \end{scope}
    
    \begin{scope}[xshift=6.75cm]
        \node[font=\tiny, anchor=east, rotate=90] at (-0.3, 2.0) {$\alpha$-entmax weights};
        \draw[->] (0, 0.5) -- (5.2, 0.5);
        \draw[->] (0, 0.5) -- (0, 2.15);
        
        \foreach \x/\h in {0.2/1.25, 0.5/0.95, 0.8/0.6, 1.1/0.35, 1.4/0.15, 1.7/0.05} {
            \fill[entmaxgreen!70] (\x-0.1, 0.5) rectangle (\x+0.1, 0.5+\h);
        }
        \foreach \x in {2.0, 2.3, 2.6, 2.9, 3.2, 3.5, 3.8, 4.1, 4.4, 4.7} {
            \draw[dropgray!50] (\x, 0.5) -- (\x, 0.52);
        }
        
        \draw[dashed, thick, dropgray] (2.15, 0.4) -- (2.15, 2.15);
        \node[font=\tiny, text=dropgray] at (2.15, 0.25) {$k$};
        \fill[keepgreen, opacity=0.3] (0, 0.5) rectangle (2.15, 2.15);
        \node[font=\scriptsize, text=entmaxgreen!80!black] at (3.6, 1.25) {Dropped ($\delta^{\text{ent}} = 0$)};
        
        \draw[decorate, decoration={brace,mirror,amplitude=4pt, raise=2pt}] (0.1, 0.45) -- (1.8, 0.45);
        \node[font=\tiny] at (0.95, 0.15) {Support $\simark$};
        
        \node[font=\tiny] at (4.6, 0.3) {Sorted tokens};
    \end{scope}
    
\end{tikzpicture}
\vspace*{-0.25cm}
\caption{Softmax assigns non-zero probability mass to all tokens, so truncation always discards some mass ($\delta^{\text{sm}} > 0$). In contrast, $\alpha$-entmax concentrates probability on a sparse support $\simark$; when $\simark$ falls within the budget, $\delta^{\text{ent}} = 0$. As budget increases, the $\alpha$-entmax error drops to zero once $k \geq |\simark|$, while the softmax error remains positive for any $k$.}
\label{fig:concept}
\end{figure}

Recent work has shown that entmax is both useful for long-context modeling and practical to train at scale \citep{vasylenko2025long,goncalves2025adasplash,gonccalves2026adasplash}. However, these advances do not address the autoregressive decoding bottleneck: full entmax decoding still reads the full KV cache before discovering which tokens receive zero attention weight. We introduce \methodname, an entmax-native sparse decoding framework that exploits sparsity before KV pages are loaded. Concretely, \methodname is a paged selective decoding method for $\alpha$-entmax attention: it scores KV pages using lightweight metadata, selects candidate pages, and computes entmax attention only over the selected KV entries. Overall, our main contributions are:\looseness=-1

\begin{itemize}
    \item We formulate entmax-native sparse decoding as support recovery: unlike softmax truncation, which always drops nonzero mass, entmax sparse decoding is exact whenever the selected candidates contain the true support.

    \item We propose a Gaussian-aware entmax selector that estimates the entmax threshold from lightweight page statistics and adapts the candidate budget to the score distribution.

    \item We show empirically that \methodname retains more support tokens, drops less probability mass, and achieves lower output error than softmax-based sparse decoding while also being faster on long sequences.
\end{itemize}

\section{Background}
\label{sec:background}

\subsection{Softmax Attention}

Let $\bm{q}\in\R^d$ be the query vector, and let $\{\bm{k}_j,\bm{v}_j\}_{j=1}^{n}$ denote the keys and values stored in the KV cache for the $n$ previous tokens. The single-head attention output is computed as follows:
\begin{align}
    s_j(\bm{q})
    =
    \frac{\bm{q}^{\top}\bm{k}_j}{\sqrt d},
    \qquad
    \bm{p}
    =
    \pi(\bm{s}(\bm{q})),
    \qquad
     \bm{o}
    =
    \sum_{j=1}^{n}
    p_j\bm{v}_j,
\end{align}

where $\pi:\R^n\to\triangle_n$ is a probability transformation.
Full decoding reads $O(n)$ keys and values per step, making
long-context attention memory-bound. Softmax assigns positive
probability to every finite score, although attention mass often
concentrates on a small subset of tokens~\citep{zhang2023ho,li2024snapkv}.
Sparse decoding exploits this concentration by reading only
selected KV entries.\looseness=-1

\subsection{The \texorpdfstring{$\alpha$}{alpha}-entmax Transformation}

$\alpha$-entmax~\citep{peters-etal-2019-sparse,blondel-entmax} is a differentiable transformation that generalizes softmax while allowing exactly sparse probability distributions. For an input score vector $\bm{s}\in\R^n$ and $\alpha>1$, $\alpha$-entmax has the form
\begin{align}
    \entmax(\bm{s})_i
    &=
    \left[
        (\alpha-1)s_i-\tau
    \right]_+^{1/(\alpha-1)},
    \label{eq:entmax-def}
\end{align}
where $[x]_+=\max\{x,0\}$ and $\tau\in\R$ is chosen so that the output sums to one. The support of the resulting distribution is
\begin{align}
    \simark
    &=
    \supp(\entmax(\bm{s}))
    =
    \left\{
        i:
        (\alpha-1)s_i>\tau
    \right\}.
    \label{eq:entmax-support}
\end{align}
Special cases include softmax as $\alpha\to1$ and sparsemax at
$\alpha=2$~\citep{martins2016softmax}; see Appendix~\ref{app:entmax}.
Tokens outside $\simark$ receive zero probability max and don't contribute to the output, so selecting a superset
of $\simark$ preserves full-cache entmax attention exactly
(Section~\ref{sec:theory}). Section~\ref{sec:method} develops
page-selection policies that exploit this property before loading KV.

\section{Theoretical Analysis}
\label{sec:theory}

We analyze sparse decoding independently of implementation details such as paging, tiling, or kernel design. Our goal is to identify what controls the approximation error when attention is computed over only a subset of tokens, and why $\alpha$-entmax differs fundamentally from softmax in this regime. Throughout this section, let $\bm{p}\in\triangle_n$ denote the full attention distribution over $n$ tokens, let $\bm{v}_1,\ldots,\bm{v}_n\in\mathbb{R}^{d_v}$ denote value vectors, and assume $\|\bm{v}_i\|_2\le B$ for all $i$.

\subsection{Approximation Error Under Truncation}
\label{sec:approx-truncation}

Consider a selection strategy that keeps a subset $\Ikeep\subseteq[n]$ and drops its complement $\Idrop=[n]\setminus\Ikeep$. Let
\begin{align}
    \delta
    &=
    \sum_{i\in\Idrop} p_i
    \label{eq:delta_approx}
\end{align}
denote the dropped probability mass. The truncated and renormalized attention distribution is
\begin{align}
    \tilde{p}_i
    &=
    \begin{cases}
        p_i/(1-\delta), & i\in\Ikeep,\\
        0, & i\in\Idrop.
    \end{cases}
\end{align}
The full and approximate attention outputs are
\begin{align}
    \bm{o}
    =
    \sum_{i=1}^{n}
    p_i\bm{v}_i,
    \quad \text{and} \quad
    \tilde{\bm{o}}
    =
    \sum_{i=1}^{n}
    \tilde{p}_i\bm{v}_i.
\end{align}

\begin{proposition}[Approximation error under truncation]
\label{prop:truncation-bound}
Let $\delta > 0$. If $\|\bm{v}_i\|_2\le B$ for all $i$, then
\begin{align}
    \|\bm{o}-\tilde{\bm{o}}\|_2
    &\le
    2B\delta.
\end{align}
Moreover, the constant is tight: equality can be attained with two tokens whose value vectors have norm $B$ and point in opposite directions.
\end{proposition}

The proof is given in Appendix~\ref{app:approx-truncation}. Proposition~\ref{prop:truncation-bound} shows that truncated attention error is controlled by the mass assigned to discarded tokens. The identity of the dropped tokens matters only through their total probability mass. Thus, a sparse decoding method should select candidate tokens that make $\delta$ small.\looseness=-1

\subsection{Exact Sparse Attention with entmax}
\label{sec:exact-sparse-attention}

For softmax, $\delta>0$ whenever any token is dropped, because every token receives strictly positive probability. Therefore, exact sparse decoding under softmax is only possible when $\Ikeep=[n]$.
In contrast, $\alpha$-entmax produces exact zeros, which can lead to full support recovery under truncation.

\begin{proposition}[Exact sparse entmax attention]
\label{prop:exact-sparse}
Let $\bm{p}=\entmax(\bm{s})$ with support $\simark$. If the kept set satisfies $\simark\subseteq\Ikeep$, then $\delta=0$ and $\tilde{\bm{o}}=\bm{o}$ exactly.
\end{proposition}

The proof is given in Appendix~\ref{app:exact-sparse}. This proposition is the key difference between softmax and entmax sparse decoding. For softmax, sparse decoding is always a dense-tail approximation. For entmax, sparse decoding becomes a support-recovery problem. In other words, once the selected candidates contain the support, the sparse output becomes fully exact.

\paragraph{Partial support capture.}
In practice, a candidate set may not capture the entire entmax support. To separate support recovery from probability-mass recovery, we report the support retention ratio
\begin{align}
    \rho
    &=
    \frac{
        |\simark\cap\Ikeep|
    }{
        |\simark|
    }.
    \label{eq:rho_approx}
\end{align}
The ratio $\rho$ measures how many support tokens are kept, whereas $\delta$ measures how much probability mass is dropped. We note that Proposition~\ref{prop:truncation-bound} depends on $\delta$, not directly on $\rho$. Thus, a method can miss some low-mass support tokens and still have small output error.

Furthermore, the advantage of entmax over softmax under the same kept set can be expressed through the dropped mass. Let $\bm{p}^{\mathrm{sm}} = \softmax(\bm{s})$ and $\bm{p}^{\mathrm{ent}} = \entmax(\bm{s})$, and let $\simark=\supp(\bm{p}^{\mathrm{ent}})$ be the entmax support. 
For the same dropped set $\Idrop$, the difference in dropped mass is
\begin{align}
    \delta^{\mathrm{sm}}
    -
    \delta^{\mathrm{ent}}
    &=
    \underbrace{
    \sum_{i\in\Idrop\cap\simark}
    \left(
        p_i^{\mathrm{sm}}
        -
        p_i^{\mathrm{ent}}
    \right)
    }_{\text{support-miss term}}
    +
    \underbrace{
    \sum_{i\in\Idrop\setminus\simark}
    p_i^{\mathrm{sm}}
    }_{\text{softmax tail term}}.
    \label{eq:advantage-decomp}
\end{align}
The second term is always nonnegative and captures probability mass that softmax assigns to tokens outside the entmax support. The first term, however, has no fixed sign and depends on the input score distribution.
Thus, entmax's advantage is largest when the candidate set captures the entmax support, in which case the support-miss term vanishes and the dropped entmax mass becomes zero.

\section{\methodname}
\label{sec:method}

The theory above suggests that sparse decoding with $\alpha$-entmax should recover the entmax support while avoiding unnecessary KV-cache reads. \methodname implements this principle with a paged decoding abstraction, following the block-based memory layout popularized by PagedAttention and vLLM~\citep{kwon2023efficient}.  
It scores pages using lightweight metadata, selects candidate pages, and computes $\alpha$-entmax attention only over the selected KV entries.

\subsection{Pipeline Overview}
\label{sec:pipeline}

At each decoding step and for each attention head, \methodname receives a query $\bm{q}\in\mathbb{R}^d$ and a KV cache containing $n$ previous tokens. We partition the cache into pages of size $P$, yielding $M=\lceil n/P\rceil$ pages. Let $\mathcal{P}_p\subseteq[n]$ denote the token indices in page $p$. The scaled attention score for token $j$ is
$
    s_j(\bm{q})
    =
    \bm{q}^{\top}\bm{k}_j / \sqrt d.
$
The selector chooses a set of pages $\mathcal{C}_{\mathrm{page}}(\bm{q}) \subseteq \{1,\ldots,M\}$ with corresponding token set $\mathcal{C}_{\mathrm{tok}}(\bm{q}) = \bigcup_{p\in\mathcal{C}_{\mathrm{page}}(\bm{q})} \mathcal{P}_p$.
Entmax attention is then computed over $\mathcal{C}_{\mathrm{tok}}(\bm{q})$:
\begin{align}
    \tilde{\bm{p}}
    &=
    \entmax
    \left(
        \{s_j(\bm{q}) : j\in\mathcal{C}_{\mathrm{tok}}(\bm{q})\}
    \right),
    \\
    \tilde{\bm{o}}
    &=
    \sum_{j\in\mathcal{C}_{\mathrm{tok}}(\bm{q})}
    \tilde{p}_j\bm{v}_j.
\end{align}
We set $\tilde{p}_j=0$ for $j\notin\mathcal{C}_{\mathrm{tok}}(\bm{q})$. If $\mathcal{C}_{\mathrm{tok}}(\bm{q})$ contains the full-cache entmax support, then Proposition~\ref{prop:exact-sparse} implies that $\tilde{\bm{o}}$ is exactly equal to the full-cache entmax output.

In terms of efficiency, the dominant cost is loading selected KV pages from memory. In \methodname, we obtain $O\left(|\mathcal{C}_{\mathrm{page}}|Pd\right) = O\left(|\mathcal{C}_{\mathrm{tok}}|d\right)$
instead of $O(nd)$ for dense decoding. Page metadata scoring costs $O(Md)=O((n/P)d)$ and does not load values. Thus, when $|\mathcal{C}_{\mathrm{tok}}|\ll n$, \methodname drastically reduces the dominant KV-cache memory traffic during autoregressive inference.

\subsection{Query-Aware Page Scoring}
\label{sec:page-scoring}

On top of the page-based memory layout, we use query-aware page scoring in the spirit of Quest~\citep{tang2024quest}. Each page stores lightweight metadata that can be evaluated before loading the full keys and values.
Concretely, for each page $p$, we store coordinate-wise minimum and maximum key vectors,
\begin{align}
    \bm{k}_{\min}^{(p)}
    =
    \min_{j\in\mathcal{P}_p}
    \bm{k}_j,
    \qquad
    \bm{k}_{\max}^{(p)}
    =
    \max_{j\in\mathcal{P}_p}
    \bm{k}_j,
\end{align}
where the min and max are taken element-wise. Given a query $\bm{q}$, the maximum possible score of any token in page $p$ is upper-bounded by
\begin{align}
    \bar{s}_{\mathrm{box}}^{(p)}(\bm{q})
    &=
    \frac{1}{\sqrt d}
    \sum_{i=1}^{d}
    \max
    \left(
        q_i k_{\min,i}^{(p)},
        q_i k_{\max,i}^{(p)}
    \right).
    \label{eq:box-page-bound}
\end{align}
For the Gaussian-aware entmax selector (presented next), we additionally store first and second coordinate-wise moments, along with their associated standard deviation:
\begin{align}
    \bm{k}_{\mathrm{avg}}^{(p)}
    =
    \frac{1}{|\mathcal{P}_p|}
    \sum_{j\in\mathcal{P}_p}
    \bm{k}_j,
    \qquad
    \bm{m}_2^{(p)}
    =
    \frac{1}{|\mathcal{P}_p|}
    \sum_{j\in\mathcal{P}_p}
    \bm{k}_j\odot\bm{k}_j,
    \qquad
    \bm{k}_{\mathrm{std}}^{(p)}
    =
    \sqrt{
        \bm{m}_2^{(p)}
        -
        \bm{k}_{\mathrm{avg}}^{(p)}
        \odot
        \bm{k}_{\mathrm{avg}}^{(p)}
    }.
\end{align}

\subsection{Candidate Selection Policies}
\label{sec:candidate-selection}

Given the page scores $\{\bar{s}_{\mathrm{box}}^{(p)}(\bm{q})\}_{p=1}^{M}$, \methodname selects a subset of pages
$\mathcal{C}_{\mathrm{page}}(\bm{q})\subseteq\{1,\ldots,M\}$. We propose two different policies for page selection.

\paragraph{Top-$k$ page selection.}
The fixed-budget policy selects the $k$ pages with largest page scores:
\begin{align}
    \mathcal{C}_{\mathrm{page}}^{\mathrm{top}\text{-}k}(\bm{q})
    &=
    \operatorname{TopK}_{k}
    \left(
        \left\{
            \bar{s}_{\mathrm{box}}^{(p)}(\bm{q})
        \right\}_{p=1}^{M}
    \right).
\end{align}
Here $\operatorname{TopK}_{k}$ returns the indices of the $k$ largest entries in its input. The token budget is at most $kP$.

\paragraph{Gaussian-aware entmax selection.}
Top-$k$ and top-$p$ are generic page-selection rules. Entmax gives a more specific target: select pages likely to contain tokens above the entmax threshold. For page $p$, we approximate the distribution of token scores inside the page by a scalar Gaussian random variable
$S^{(p)}\sim\mathcal{N}(\mu_p(\bm{q}),\sigma_p^2(\bm{q}))$, where
\begin{align}
    \mu_p(\bm{q})
    &=
    \frac{
        \bm{q}^{\top}
        \bm{k}_{\mathrm{avg}}^{(p)}
    }{
        \sqrt d
    },
    \qquad
    \sigma_p^2(\bm{q})
    =
    \frac{1}{d}
    \sum_{i=1}^{d}
    q_i^2
    \left(
        k_{\mathrm{std},i}^{(p)}
    \right)^2.
\end{align}
Here $S^{(p)}$ should be interpreted as the score of a randomly chosen token from page $p$. This diagonal approximation ignores cross-coordinate covariance, but it can be computed using only page metadata.
Recall that entmax probabilities have the form
\begin{align}
    g_\alpha(s;\tau)
    &=
    \left[
        (\alpha-1)s-\tau
    \right]_+^{1/(\alpha-1)},
\end{align}
where the exact threshold $\tau$ is chosen so that $\sum_{i} g_\alpha(s_i;\tau) = 1$.
Computing this sum exactly would require visiting every token score. Instead, we replace the empirical score distribution within each page by its Gaussian approximation. Under this model, the expected contribution of one token in page $p$ at threshold $\tau$ is
$\E_{S^{(p)}}[g_\alpha(S^{(p)};\tau)]$. Summing over the expected contribution of whole pages gives the approximate normalization equation
\begin{align}
    \sum_{p=1}^{M}
    |\mathcal{P}_p| \,
    \E_{S^{(p)}}
    \left[
        g_\alpha(S^{(p)};\hat{\tau})
    \right]
    &=
    1.
    \label{eq:gaussian-threshold-main}
\end{align}
Solving Eq.~\eqref{eq:gaussian-threshold-main} gives an estimate $\hat{\tau}$ of the full entmax threshold using only page-level moments. For $\alpha\in\{4/3,\,3/2,\,2\}$, the expectation has a closed form in terms of truncated Gaussian moments, as shown in Appendix~\ref{app:detg-closed-form}.\footnote{For these values, $\beta=1/(\alpha-1)$ is an integer: $\beta=1$ for $\alpha=2$, $\beta=2$ for $\alpha=3/2$, and $\beta=3$ for $\alpha=4/3$. The expectation then reduces to truncated Gaussian moments of first, second, and third order, respectively.} Each root-finding iteration costs constant time per page, since each page contributes only through $\mu_p$ and $\sigma_p$.
Given $\hat{\tau}$, we next estimate whether a page may contain at least one token above the threshold. Under the Gaussian model, the probability that all scores in page $p$ are below a value $b$ is approximately
\begin{align}
    \Pr
    \left[
        \max_{j\in\mathcal{P}_p}
        s_j
        \le
        b
    \right]
    &\approx
    \Phi
    \left(
        \frac{b-\mu_p}{\sigma_p}
    \right)^{|\mathcal{P}_p|}.
\end{align}
Therefore, the value
\begin{align}
    \bar{s}_{\mathrm{Gauss}}^{(p)}
    &=
    \mu_p
    +
    \sigma_p
    \Phi^{-1}
    \left(
        q_{\mathrm{page}}^{1/|\mathcal{P}_p|}
    \right)
    \label{eq:gaussian-page-bound-main}
\end{align}
is a $q_{\mathrm{page}}$-confidence estimate of the maximum score in page $p$.
The exponent $1/|\mathcal{P}_p|$ appears because the confidence is for the maximum over the entire page, not for a single token.
Finally, we select pages whose estimated maximum is above the estimated entmax threshold:
\begin{align}
    \mathcal{C}_{\mathrm{page}}^{\mathrm{Gauss}}(\bm{q})
    &=
    \left\{
        p:
        (\alpha-1)
        \bar{s}_{\mathrm{Gauss}}^{(p)}
        >
        \hat{\tau}
        -
        \Delta
    \right\}.
    \label{eq:gaussian-selector-main}
\end{align}
Here, $\Delta\ge0$ is a safety margin hyperparameter. Larger $\Delta$ selects more pages and reduces false negatives, whereas lower values make the estimation more precise. We note that unlike the deterministic box bound in Eq.~\eqref{eq:box-page-bound}, $\bar{s}_{\mathrm{Gauss}}^{(p)}$ is not a worst-case upper bound; instead, it is a model-based estimate that aligns page selection with the entmax support condition $(\alpha-1)s_i>\tau$. 

\subsection{Triton Kernel}
\label{sec:triton}

Our Triton kernels operate directly on a paged KV cache,
reading selected pages through their indices.
The top-$k$ variant uses Quest-style metadata
scoring~\citep{tang2024quest}, selects pages, and computes entmax
over their tokens using histogram-based threshold initialization and Halley refinement, as done by AdaSplash-2~\citep{gonccalves2026adasplash}.
The Gaussian variant computes score moments per head and batch element, solves Eq.~\eqref{eq:gaussian-threshold-main}, and selects
pages using Eq.~\eqref{eq:gaussian-selector-main}.
It passes the selected indices and estimated threshold to the
decode kernel, which performs one additional Halley refinement
on the selected scores before computing the output.
Both selectors operate on page metadata before reading selected KV. More details about the Triton implementation and our sparse decoding kernels can be found in Appendix ~\ref{app:implementation-details}.

\section{Experiments}
\label{sec:experiments}

Efficiency benchmarks are executed on synthetic and real data while the the task-based benchmarks use the 1B-parameter and 32K context length softmax and entmax NAPE checkpoints
from AdaSplash-2~\citep{gonccalves2026adasplash}. We provide details about the experimental setup in Appendix~\ref{app:experimental_setup}. 
Moreover, we emphasize that every time we refer to softmax top-$k$, it means exactly Quest \cite{tang2024quest} ported to Triton and our model architecture.\footnote{Porting Quest to Triton and adapting it to our model architecture ensures a fair comparison. Otherwise, our Triton implementation would be orders of magnitude faster than Quest's original implementation.}

\subsection{Approximation Error Analysis}
\label{sec:approximation-error-experiments}

\begin{figure}[t]
    \centering
    \includegraphics[width=0.65\linewidth]{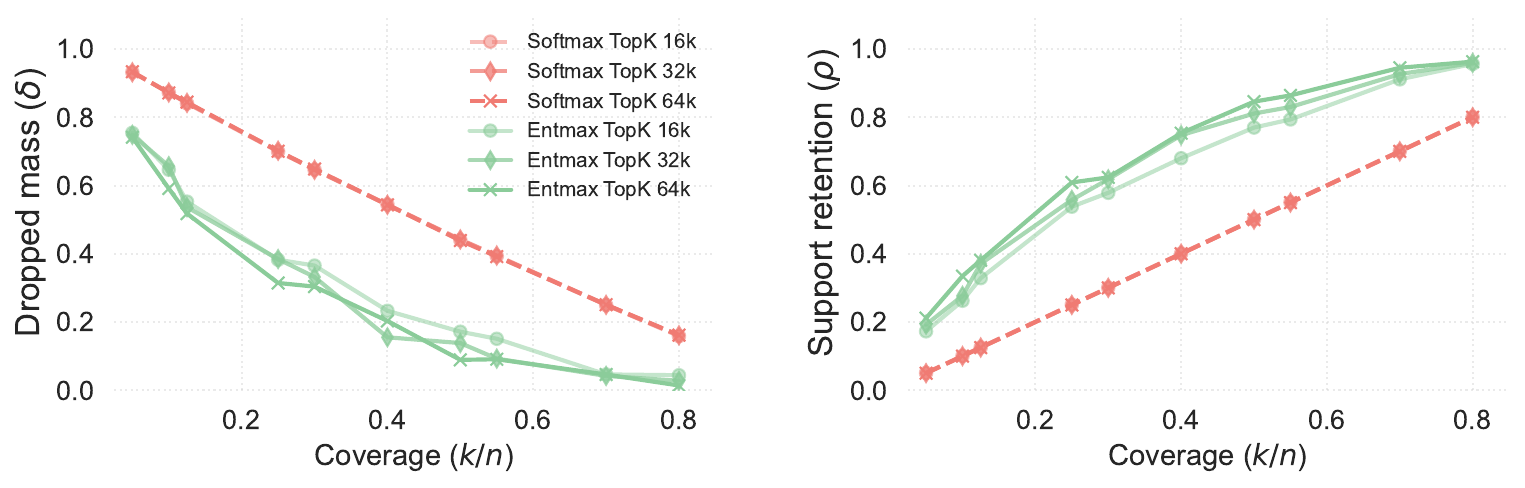}
    \hfill
    \includegraphics[width=0.335\linewidth]{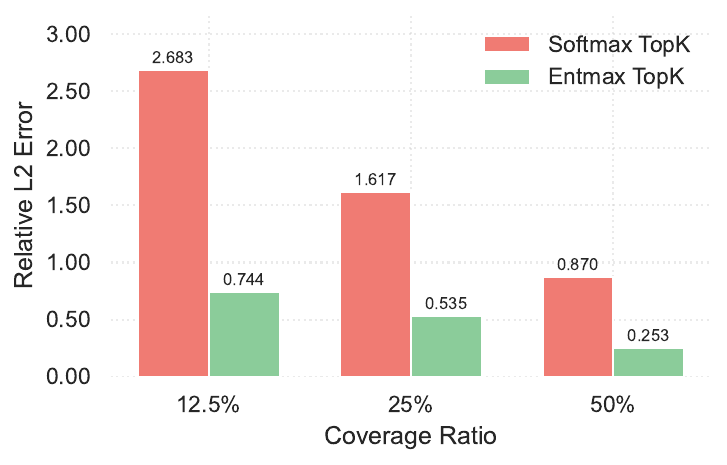}
    \caption{Attention approximation quality in terms of the dropped probability mass $\delta$ (left), support retention $\rho$ (middle),
    and relative output error $\|\bm{o}-\tilde{\bm{o}}\|_2/\|\bm{o}\|_2$ (right).
    Coverage ratio denotes the fraction of the KV cache used in the attention computation. 
    }
    \label{fig:relative_error}
\end{figure}

Figure~\ref{fig:relative_error} compares sparse softmax and sparse entmax at matched coverage ratios. 
Entmax-native selection consistently drops less probability mass and retains a larger percentage of the support, with the largest gains appearing in the most aggressive budget regimes. This behavior is reflected in the output error: sparse entmax produces substantially smaller relative error than sparse softmax at the same coverage. These results agree with Proposition~\ref{prop:truncation-bound}, which shows that truncation error is controlled by the dropped mass $\delta$. They also support the support-recovery interpretation of sparse entmax decoding: once the selected pages contain the entmax support, the approximation error disappears.  This advantage also holds for real activations from the model, as shown in Appendix \ref{app:real_activations_quality} where we can see that the Gaussian-Aware method shows near perfect accuracy with high sparsity and an adaptive budget. Table \ref{tab:real-activations-topk} in App.~\ref{subsubsec:Real activations} shows the comparison between softmax and entmax top-$k$, with entmax consistently dropping less probability mass, retaining a higher percentage of the support, and achieving lower relative error than softmax top-$k$.

\subsection{Passkey Retrieval}
\label{sec:passkey}

\begin{figure}[t]
    \centering
    \includegraphics[width=1\textwidth]{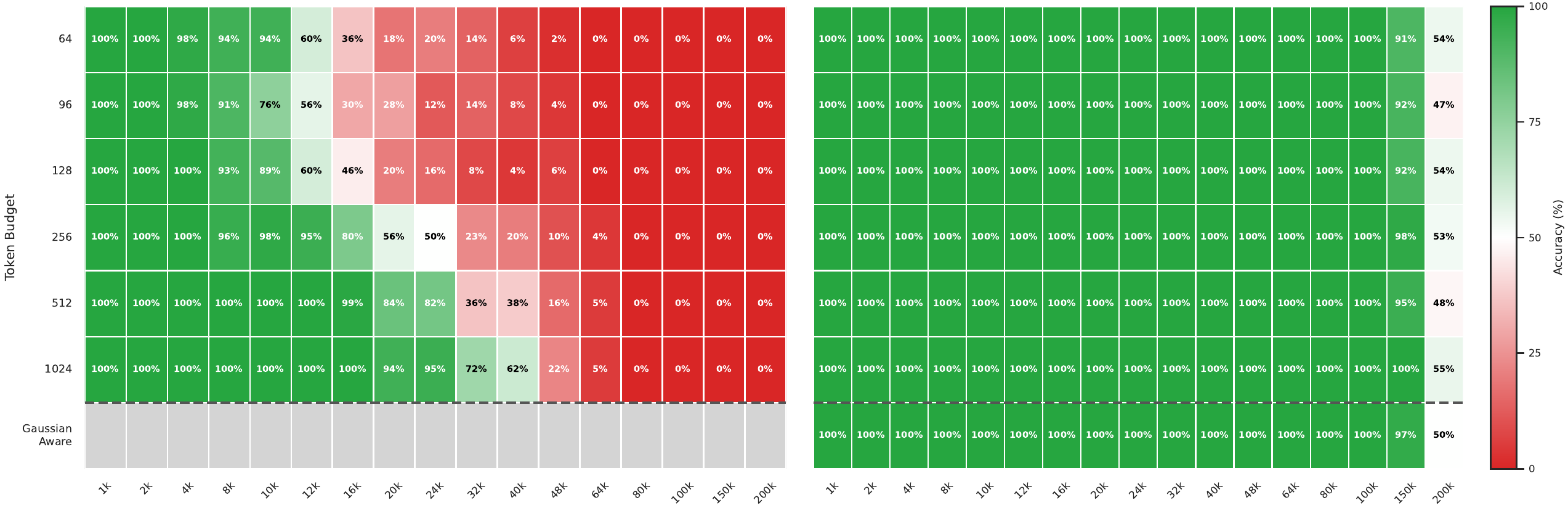}
    \caption{Passkey retrieval accuracy across KV-cache lengths ($x$) and token budgets ($y$). The left panel shows softmax top-$k$ sparse decoding; the right part shows entmax top-$k$ and Gaussian-aware variants.\looseness=-1}
    \label{fig:passkey_heatmap}
\end{figure}

Figure~\ref{fig:passkey_heatmap} shows a clear separation between softmax and entmax under small sparse budgets. Softmax sparse attention degrades rapidly as the sequence length grows, even when the selected token budget is increased. In contrast, entmax-based sparse decoding maintains perfect retrieval across substantially longer contexts, in some cases using less than $0.1\%$ of the KV cache. These results support the central hypothesis of the paper: because entmax attention has finite support, sparse decoding can remain exact or near-exact when the selected pages cover that support. 
Moreover, both entmax variants begin to degrade only at around 200k context tokens, roughly 6x the model's training context length. By contrast, the softmax approaches degrade even at in-domain lengths of 12k–20k tokens and typically require substantially larger budgets to recover.

\subsection{Language Modeling Perplexity}
\label{sec:perplexity}

\begin{wrapfigure}{r}{0.4\linewidth}
    \vspace*{-1.3cm}
    \centering
    \includegraphics[width=\linewidth]{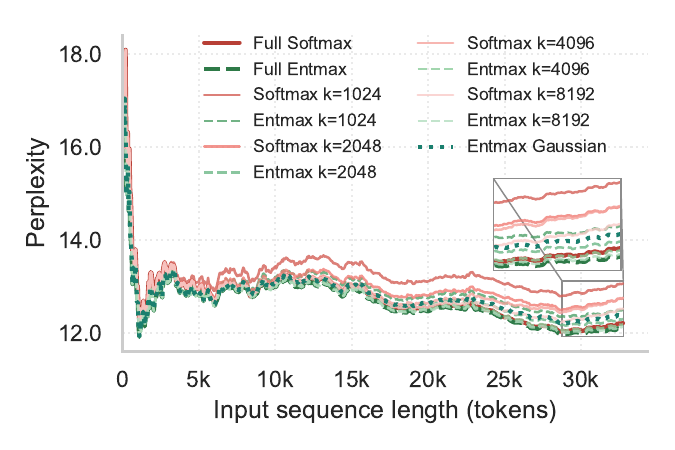}
    \vspace*{-0.7cm}
    \caption{Perplexity on PG19. 
    }
    \label{fig:perplexity}
\end{wrapfigure}

Figure~\ref{fig:perplexity} shows that entmax sparse decoding remains close to full-cache entmax even at aggressive budgets while softmax sparse decoding incurs a larger relative degradation. Notably, entmax with a 1024-token budget achieves lower absolute perplexity than softmax with substantially larger budgets. This suggests that entmax sparsity not only improves approximation quality at the attention level, but also translates into stronger language-modeling behavior under sparse decoding.

\subsection{RULER Benchmarks}
\label{sec:ruler}

\begin{table*}[t]
\caption{RULER scores across context lengths ($n$).
We use 1024 tokens fox fixed budget variants;
Gaussian-aware selection uses a variable budget.
Bold and underline indicate the best and second-best average within each context length. Softmax top-$k$ refers to Quest~\citep{tang2024quest}.
}
\label{tab:ruler_results}
\centering
\small
\setlength{\tabcolsep}{2.7pt}
\resizebox{\linewidth}{!}{%
\begin{tabular}{ll*{14}{r}}
\toprule
$n$ & Method &
mk1 & mk2 & mk3 & mq & mv &
niah1 & niah2 & niah3 &
cwe & fwe & vt & qa1 & qa2 & avg \\
\midrule
4k & Softmax
& 84.4 & 39.8 & 10.0 & 20.4 & 16.3 & 100.0 & 100.0 & 100.0 & 28.0 & 30.3 & 14.3 & 22.6 & 37.3 & 46.4 \\
4k & Softmax (top-$k$)
& 79.6 & 37.8 & 5.2 & 35.2 & 22.0 & 100.0 & 100.0 & 99.8 & 10.8 & 24.6 & 9.5 & 22.4 & 38.8 & 45.1 \\
4k & Softmax (Twilight)
& 76.4 & 36.2 & 6.2 & 37.2 & 24.8 & 100.0 & 100.0 & 88.4 & 10.7 & 25.0 & 21.1 & 23.0 & 38.0 & 45.1 \\
4k & Softmax (SnapKV)
& 9.0 & 0.6 & 1.8 & 5.9 & 5.6 & 100.0 & 10.4 & 2.2 & 57.5 & 58.2 & 27.84 & 26.0 & 26.67 & 25.5 \\
4k & Entmax
& 81.8 & 13.0 & 29.6 & 56.3 & 60.6 & 100.0 & 100.0 & 76.4 & 48.5 & 58.1 & 37.7 & 29.8 & 37.4 & \textbf{56.1} \\
4k & Entmax (top-$k$)
& 81.4 & 21.4 & 22.6 & 60.1 & 63.4 & 100.0 & 100.0 & 75.6 & 40.4 & 57.9 & 36.8 & 28.4 & 36.4 & \underline{55.7} \\
4k & Entmax (Gaus.)
& 82.2 & 12.8 & 28.6 & 56.1 & 62.4 & 100.0 & 99.8 & 75.4 & 40.4 & 56.3 & 37.1 & 29.8 & 36.6 & 55.2 \\
\midrule
8k & Softmax
& 76.0 & 35.0 & 11.2 & 27.0 & 26.2 & 100.0 & 100.0 & 98.8 & 16.4 & 23.2 & 10.4 & 20.8 & 20.7 & 43.5 \\
8k & Softmax (top-$k$)
& 70.2 & 33.0 & 1.4 & 29.3 & 25.0 & 100.0 & 93.6 & 86.4 & 5.5 & 23.9 & 7.0 & 19.6 & 15.3 & 39.3 \\
8k & Softmax (Twilight)
& 51.0 & 11.0 & 0.0 & 19.8 & 21.1 & 89.4 & 83.8 & 12.4 & 3.6 & 10.1 & 3.3 & 17.8 & 15.4 & 26.0 \\
8k & Softmax (SnapKV)
& 25.6 & 3.6 & 0.0 & 15.85 & 14.25 & 96.8 & 31.6 & 3.2 & 14.54 & 24.07 & 0.96 & 17.4 & 13.0 & 20.1 \\
8k & Entmax
& 73.6 & 6.2 & 19.2 & 63.5 & 32.8 & 100.0 & 100.0 & 82.0 & 34.3 & 47.1 & 39.2 & 28.6 & 25.8 & \textbf{50.2} \\
8k & Entmax (top-$k$)
& 72.2 & 18.2 & 5.8 & 61.5 & 39.7 & 100.0 & 100.0 & 83.2 & 15.0 & 40.5 & 38.5 & 27.0 & 25.0 & 48.2 \\
8k & Entmax (Gaus.)
& 73.0 & 4.0 & 15.6 & 60.7 & 32.4 & 100.0 & 100.0 & 82.0 & 32.1 & 47.1 & 27.0 & 26.7 & 39.5 & \underline{49.3} \\
\midrule
16k & Softmax
& 82.6 & 28.8 & 10.4 & 22.6 & 25.8 & 100.0 & 100.0 & 98.8 & 15.0 & 23.9 & 9.7 & 19.4 & 20.8 & 42.9 \\
16k & Softmax (top-$k$)
& 54.4 & 4.0 & 0.0 & 4.3 & 11.2 & 99.0 & 88.2 & 48.6 & 2.5 & 9.2 & 2.1 & 19.0 & 16.2 & 27.6 \\
16k & Softmax (Twilight)
& 11.6 & 0.8 & 0.0 & 2.2 & 2.3 & 34.6 & 27.0 & 0.0 & 1.2 & 2.1 & 1.0 & 10.0 & 8.7 & 7.8 \\
16k & Softmax (SnapKV)
& 17.2 & 1.8 & 0.2 & 8.1 & 7.3 & 76.4 & 18.9 & 1.1 & 8.3 & 13.7 & 1.2 & 13.8 & 11.5 & 13.8 \\
16k & Entmax
& 77.8 & 1.2 & 9.4 & 37.8 & 22.4 & 100.0 & 100.0 & 71.4 & 23.0 & 49.3 & 35.0 & 27.6 & 27.7 & \textbf{44.8} \\
16k & Entmax (top-$k$)
& 73.4 & 8.4 & 1.2 & 33.0 & 20.6 & 100.0 & 98.8 & 71.2 & 4.4 & 42.4 & 33.4 & 23.6 & 25.6 & 41.2 \\
16k & Entmax (Gaus.)
& 75.4 & 2.6 & 6.2 & 37.9 & 21.7 & 100.0 & 99.8 & 72.8 & 23.5 & 48.0 & 34.32 & 26.8 & 27.0 & \underline{44.3} \\
\midrule
32k & Softmax
& 56.6 & 5.4 & 5.0 & 21.4 & 25.2 & 100.0 & 84.4 & 96.0 & 7.1 & 3.26 & 6.5 & 20.4 & 20.4 & 34.7 \\
32k & Softmax (top-$k$)
& 25.0 & 0.6 & 0.0 & 3.9 & 5.5 & 40.4 & 12.8 & 1.0 & 0.9 & 0.3 & 0.2 & 12.4 & 11.2 & 8.78 \\
32k & Softmax (Twilight)
& 2.4 & 0.2 & 0.0 & 1.5 & 2.1 & 13.8 & 2.0 & 0.0 & 0.5 & 0.1 & 2.3 & 6.8 & 6.4 & 2.9 \\
32k & Softmax (SnapKV)
& 5.8 & 0.4 & 0.0 & 2.7 & 3.1 & 27.2 & 4.6 & 0.2 & 2.1 & 3.5 & 0.4 & 8.2 & 7.0 & 5.0 \\
32k & Entmax
& 59.4 & 2.4 & 4.6 & 26.1 & 16.1 & 100.0 & 90.6 & 71.2 & 11.6 & 42.1 & 27.8 & 25.4 & 23.1 & \underline{38.5} \\
32k & Entmax (top-$k$)
& 69.8 & 4.8 & 0.0 & 30.3 & 16.9 & 100.0 & 95.0 & 75.4 & 1.8 & 38.0 & 31.4 & 25.8 & 19.2 & 36.2 \\
32k & Entmax (Gaus.)
& 70.8 & 1.8 & 0.6 & 33.5 & 17.5 & 100.0 & 98.4 & 76.4 & 12.6 & 45.9 & 31.2 & 25.8 & 20.9 & \textbf{41.2} \\
\bottomrule
\end{tabular}
}
\end{table*}

Table~\ref{tab:ruler_results} shows that entmax outperforms softmax across context lengths, and that the gap is especially pronounced under sparse decoding. At 16k and 32k softmax degrades sharply while entmax with the same budget remains substantially closer to its dense performance. This indicates that the support-recovery behavior observed in the approximation metrics carries over to downstream long-context tasks. Both entmax variants closely match the baseline performance, moreover, the Gaussian variant is able to avoid the CWE/FWE degradation inherent to top-$k$ approaches shown in MagicPIG \citep{chen2024magicpig}. Twilight \cite{lin2025twilight} combines a coarse top-$k$ selection with top-$p$ selection and 4-bit K cache quantization for greater efficiency but we found that it was unable to prevent softmax top-$k$'s degradation in our experiments. We also included a comparison with SnapKV\cite{li2024snapkv}, even though it is a KV cache cache compression method, not a KV cache selection method, so it incurs a higher risk of losing information.

\subsection{Efficiency Benchmarks}
\label{sec:efficiency}

Fig.~\ref{fig:wall_time_speedups} reports end-to-end attention latency normalized by full-cache softmax implemented in FlashDecoding~\citep{dao2023flashdecoding}. 
Full entmax becomes increasingly expensive with context length, taking twice as much as FlashAttention in Triton at 1M tokens. Entmax top-$k$ closes the gap with softmax top-$k$ (Quest) by avoiding full-cache KV reads and leveraging and optimized paged kernel.
The Gaussian-aware entmax selector scales best among the entmax variants. 
Specifically, while its distributional selection stage adds overhead at shorter lengths, it avoids the explicit top-$k$ bottleneck and grows slower than the other methods with sequence length. At 1M tokens, it runs in $3.6$ms, achieving a $2.88\times$ speedup over full-cache softmax FlashDecoding and a $5.5\times$ speedup over the full entmax reference. Overall, \methodname shifts decoding cost from full-cache memory traffic to lightweight page selection and sparse paged entmax computation while twice as fast as both top-$k$ methods.

\begin{figure}[t]
    \centering
    \includegraphics[width=1\linewidth]{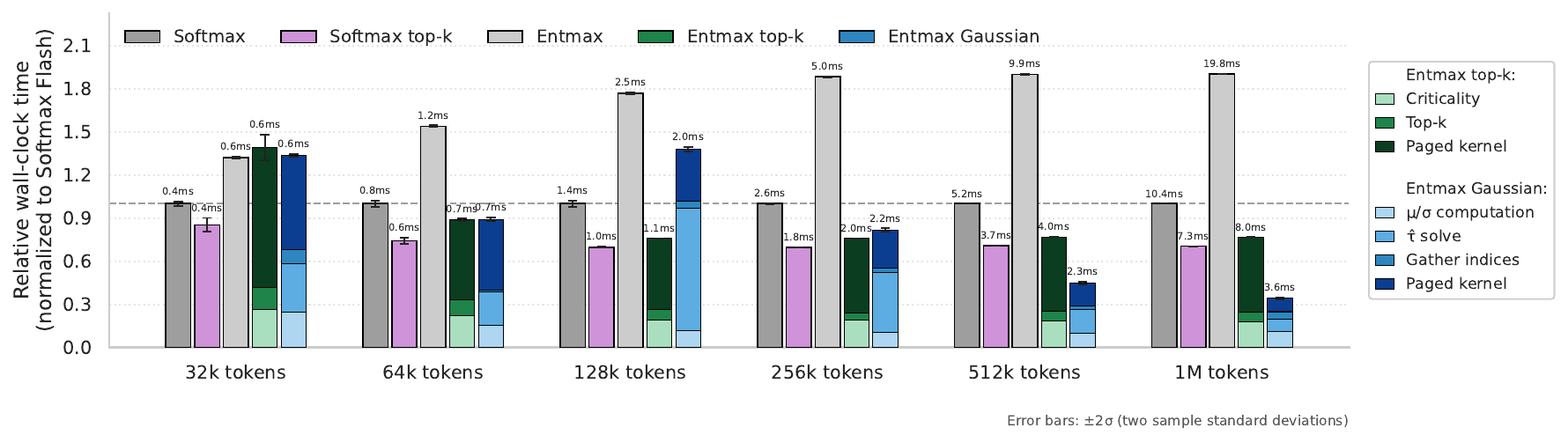}
    \caption{Wall-clock time speedup relative to softmax (FlashDecoding) at different KV-cache lengths. Top-$k$ methods use 25\% of the total length.
    We also report the actual wall-clock time (in ms) on top of each bar for each method.}
    \label{fig:wall_time_speedups}
\end{figure}

\section{Related Work}
\label{sec:related_works}

\paragraph{KV-cache compression and sparse decoding.}
Several methods reduce KV-cache cost by evicting or compressing tokens deemed unimportant. H2O evicts tokens based on cumulative attention scores~\citep{zhang2023ho}; StreamingLLM preserves attention sinks and a sliding window~\citep{xiao2024efficient}; TOVA removes tokens according to current-query attention behavior~\citep{oren-etal-2024-transformers}; and SnapKV compresses prompts by selecting important KV states~\citep{li2024snapkv}. These methods reduce memory footprint, but risk losing information needed by future queries.
A complementary line of work keeps the full cache available and selects a query-dependent subset at decoding time. Quest uses page-level min/max key bounds to identify critical pages without scanning the full cache~\citep{tang2024quest}. 
Twilight extends query-aware selection by first forming a top-$k$ candidate pool and then applying top-$p$ filtering within that pool~\citep{lin2025twilight}. Other methods use clustering, dynamic sparse patterns, or compression-based approximations to reduce attention computation during inference~\citep{singhania2024loki,chen2024magicpig,nawrot2024dynamic,zhu2025tactic,devoto-etal-2024-simple}. Our work belongs to this query-aware family, but differs in that the selection objective is tailored to $\alpha$-entmax support recovery rather than softmax tail approximation.

\paragraph{Sparse probability mappings and entmax attention.}
Sparsemax and $\alpha$-entmax are alternatives to softmax that can assign exactly zero probability to low-scoring entries~\citep{martins2016softmax,peters-etal-2019-sparse,blondel-entmax}. Adaptively sparse transformers learn or tune the sparsity parameter across heads, allowing attention distributions to interpolate between dense and sparse regimes~\citep{correia-etal-2019-adaptively}. ASentmax introduces input-dependent learnable temperatures and argues that sparse attention improves length generalization by mitigating attention dispersion, representational collapse, and over-squashing~\citep{vasylenko2025long}.

\paragraph{Efficient entmax kernels.}
AdaSplash introduced GPU kernels for $\alpha$-entmax attention~\citep{goncalves2025adasplash}. AdaSplash-2 further accelerates entmax training through a histogram-based initialization and efficient block skipping, showing strong long-context results on benchmarks such as RULER and HELMET~\citep{gonccalves2026adasplash}. Our work is orthogonal. While AdaSplash kernels make sparse attention practical during training, decoding still reads the full KV cache before sparsity is known.
\methodname exploits entmax sparsity during autoregressive decoding by selecting KV pages before they are loaded, reducing KV-cache memory movement while preserving exactness whenever the support is captured.\looseness=-1

\section{Conclusion}
\label{sec:conclusion}

We introduced \methodname, an entmax-native sparse decoding framework for long-context inference, combining query-aware page scoring, support-aware candidate selection, and sparse entmax attention for autoregressive decoding.
We showed that truncation error is controlled by the dropped probability mass $\delta$, and that this error vanishes when the entmax support is recovered. We also introduced a Gaussian-aware entmax selector that estimates the entmax threshold from page-level statistics, allowing the candidate budget to adapt to the score distribution.
Empirically, \methodname retains more support tokens, drops less probability mass, and achieves lower output error than softmax-based sparse decoding at matched KV budgets, while improving decoding efficiency over full entmax and dense attention baselines.

\subsection*{AI use statement}

In this work, we used generative AI tools to assist in writing the proofs and fix notation issues.
We also used generative AI tools for summarizing and analyzing existing literature, searching for information, discover new literature besides the manual investigation, and edit the writing and polishing the paper for readability.
We used generative AI tools to write and edit software code.
All uses of AI tools were accompanied by a careful review process. %
In particular, we manually checked the proofs and made them more accessible, and LLM-generated code was verified and tested for correctness.
We take responsibility for the final content of this work, including any text, claims, or artifacts produced with the aid of generative AI.

\subsection*{Reproducibility statement}

Our code is available at \url{https://github.com/deep-spin/entmaxkv}. 
All implementation and training details can be found in Appendix~\ref{app:implementation-details} and \ref{app:additional_experiments}.

\bibliography{refs}

@InProceedings{blondel-entmax,
  title = 	 {Learning Classifiers with Fenchel-Young Losses: Generalized Entropies, Margins, and Algorithms},
  author =       {Blondel, Mathieu and Martins, Andre and Niculae, Vlad},
  booktitle = 	 {Proceedings of the Twenty-Second International Conference on Artificial Intelligence and Statistics},
  pages = 	 {606--615},
  year = 	 {2019},
  editor = 	 {Chaudhuri, Kamalika and Sugiyama, Masashi},
  volume = 	 {89},
  series = 	 {Proceedings of Machine Learning Research},
  month = 	 {16--18 Apr},
  publisher =    {PMLR},
  url = 	 {https://proceedings.mlr.press/v89/blondel19a.html}
}

@inproceedings{correia-etal-2019-adaptively,
    title = "Adaptively Sparse Transformers",
    author = "Correia, Gon{\c{c}}alo M.  and
      Niculae, Vlad  and
      Martins, Andr{\'e} F. T.",
    editor = "Inui, Kentaro  and
      Jiang, Jing  and
      Ng, Vincent  and
      Wan, Xiaojun",
    booktitle = "Proceedings of the 2019 Conference on Empirical Methods in Natural Language Processing and the 9th International Joint Conference on Natural Language Processing (EMNLP-IJCNLP)",
    month = nov,
    year = "2019",
    address = "Hong Kong, China",
    publisher = "Association for Computational Linguistics",
    url = "https://aclanthology.org/D19-1223/",
    doi = "10.18653/v1/D19-1223",
    pages = "2174--2184"
}

@inproceedings{rae2019compressive,
    title={Compressive Transformers for Long-Range Sequence Modelling},
    author={Jack W. Rae and Anna Potapenko and Siddhant M. Jayakumar and Chloe Hillier and Timothy P. Lillicrap},
    booktitle={International Conference on Learning Representations},
    year={2020},
    url={https://openreview.net/forum?id=SylKikSYDH}
}

@inproceedings{kwon2023efficient,
  title={Efficient memory management for large language model serving with pagedattention},
  author={Kwon, Woosuk and Li, Zhuohan and Zhuang, Siyuan and Sheng, Ying and Zheng, Lianmin and Yu, Cody Hao and Gonzalez, Joseph and Zhang, Hao and Stoica, Ion},
  booktitle={Proceedings of the 29th Symposium on Operating Systems Principles},
  pages={611--626},
  year={2023}
}

@inproceedings{devoto-etal-2024-simple,
    title = "A Simple and Effective $L\_2$ Norm-Based Strategy for {KV} Cache Compression",
    author = "Devoto, Alessio  and
      Zhao, Yu  and
      Scardapane, Simone  and
      Minervini, Pasquale",
    editor = "Al-Onaizan, Yaser  and
      Bansal, Mohit  and
      Chen, Yun-Nung",
    booktitle = "Proceedings of the 2024 Conference on Empirical Methods in Natural Language Processing",
    month = nov,
    year = "2024",
    address = "Miami, Florida, USA",
    publisher = "Association for Computational Linguistics",
    url = "https://aclanthology.org/2024.emnlp-main.1027/",
    doi = "10.18653/v1/2024.emnlp-main.1027",
    pages = "18476--18499",
}

@inproceedings{xiao2024efficient,
    title={Efficient Streaming Language Models with Attention Sinks},
    author={Guangxuan Xiao and Yuandong Tian and Beidi Chen and Song Han and Mike Lewis},
    booktitle={The Twelfth International Conference on Learning Representations},
    year={2024},
    url={https://openreview.net/forum?id=NG7sS51zVF}
}

@article{chen2024magicpig,
  title={Magicpig: Lsh sampling for efficient llm generation},
  author={Chen, Zhuoming and Sadhukhan, Ranajoy and Ye, Zihao and Zhou, Yang and Zhang, Jianyu and Nolte, Niklas and Tian, Yuandong and Douze, Matthijs and Bottou, Leon and Jia, Zhihao and others},
  journal={arXiv preprint arXiv:2410.16179},
  year={2024}
}

@inproceedings{peters-etal-2019-sparse,
    title = "Sparse Sequence-to-Sequence Models",
    author = "Peters, Ben  and
      Niculae, Vlad  and
      Martins, Andr{\'e} F. T.",
    booktitle = "Proceedings of the 57th Annual Meeting of the Association for Computational Linguistics",
    month = jul,
    year = "2019",
    address = "Florence, Italy",
    publisher = "Association for Computational Linguistics",
    url = "https://www.aclweb.org/anthology/P19-1146",
    doi = "10.18653/v1/P19-1146",
    pages = "1504--1519",
}

@inproceedings{martins2016softmax,
  title = 	 {From Softmax to Sparsemax: A Sparse Model of Attention and Multi-Label Classification},
  author = 	 {Andre Martins and Ramon Astudillo},
  pages = 	 {1614--1623},
  year = 	 {2016},
  editor = 	 {Maria Florina Balcan and Kilian Q. Weinberger},
  volume = 	 {48},
  booktitle ={International Conference on Machine Learning (ICML)},
  series = 	 {Proceedings of Machine Learning Research},
  address = 	 {New York, New York, USA},
  month = 	 {20--22 Jun},
  publisher =    {PMLR},
  url = 	 {http://proceedings.mlr.press/v48/martins16.html}
}

@inproceedings{goncalves2025adasplash,
    title={AdaSplash: Adaptive Sparse Flash Attention},
    author={Nuno Gon{\c{c}}alves and Marcos V Treviso and Andre Martins},
    booktitle={Forty-second International Conference on Machine Learning},
    year={2025},
    url={https://openreview.net/forum?id=OWIPDWhUcO}
}

@inproceedings{
vasylenko2025long,
title={Long-Context Generalization with Sparse Attention},
author={Pavlo Vasylenko and Hugo Pitorro and Andre Martins and Marcos Vinicius Treviso},
booktitle={The Fourteenth International Conference on Learning Representations},
year={2026},
url={https://openreview.net/forum?id=PsB6Lynznk}
}

@inproceedings{
tang2024quest,
title={{QUEST}: Query-Aware Sparsity for Efficient Long-Context {LLM} Inference},
author={Jiaming Tang and Yilong Zhao and Kan Zhu and Guangxuan Xiao and Baris Kasikci and Song Han},
booktitle={Forty-first International Conference on Machine Learning},
year={2024},
url={https://openreview.net/forum?id=KzACYw0MTV}
}

@inproceedings{zhang2023ho,
title={H2O: Heavy-Hitter Oracle for Efficient Generative Inference of Large Language Models},
author={Zhenyu Zhang and Ying Sheng and Tianyi Zhou and Tianlong Chen and Lianmin Zheng and Ruisi Cai and Zhao Song and Yuandong Tian and Christopher Re and Clark Barrett and Zhangyang Wang and Beidi Chen},
booktitle={Thirty-seventh Conference on Neural Information Processing Systems},
year={2023},
url={https://openreview.net/forum?id=RkRrPp7GKO}
}

@inproceedings{oren-etal-2024-transformers,
    title = "Transformers are Multi-State {RNN}s",
    author = "Oren, Matanel  and
      Hassid, Michael  and
      Yarden, Nir  and
      Adi, Yossi  and
      Schwartz, Roy",
    editor = "Al-Onaizan, Yaser  and
      Bansal, Mohit  and
      Chen, Yun-Nung",
    booktitle = "Proceedings of the 2024 Conference on Empirical Methods in Natural Language Processing",
    month = nov,
    year = "2024",
    address = "Miami, Florida, USA",
    publisher = "Association for Computational Linguistics",
    url = "https://aclanthology.org/2024.emnlp-main.1043/",
    doi = "10.18653/v1/2024.emnlp-main.1043",
    pages = "18724--18741"
}

@inproceedings{
lin2025twilight,
title={Twilight: Adaptive Attention Sparsity with Hierarchical Top-\$p\$  Pruning},
author={Chaofan Lin and Jiaming Tang and Shuo Yang and Hanshuo Wang and Tian Tang and Boyu Tian and Ion Stoica and Song Han and Mingyu Gao},
booktitle={The Thirty-ninth Annual Conference on Neural Information Processing Systems},
year={2025},
url={https://openreview.net/forum?id=Ve693NkzcU}
}

@inproceedings{singhania2024loki,
 author = {Singhania, Prajwal and Singh, Siddharth and He, Shwai and Feizi, Soheil and Bhatele, Abhinav},
 booktitle = {Advances in Neural Information Processing Systems},
 doi = {10.52202/079017-0532},
 editor = {A. Globerson and L. Mackey and D. Belgrave and A. Fan and U. Paquet and J. Tomczak and C. Zhang},
 pages = {16692--16723},
 publisher = {Curran Associates, Inc.},
 title = {Loki: Low-rank Keys for Efficient Sparse Attention},
 url = {https://proceedings.neurips.cc/paper_files/paper/2024/file/1e027da6bec9ceb2ec37951ceeccae93-Paper-Conference.pdf},
 volume = {37},
 year = {2024}
}

@inproceedings{nawrot2024dynamic,
title={Dynamic Memory Compression: Retrofitting {LLM}s for Accelerated Inference},
author={Piotr Nawrot and Adrian {\L}a{\'n}cucki and Marcin Chochowski and David Tarjan and Edoardo Ponti},
booktitle={Forty-first International Conference on Machine Learning},
year={2024},
url={https://openreview.net/forum?id=tDRYrAkOB7}
}

@article{li2024snapkv,
  title={Snapkv: Llm knows what you are looking for before generation},
  author={Li, Yuhong and Huang, Yingbing and Yang, Bowen and Venkitesh, Bharat and Locatelli, Acyr and Ye, Hanchen and Cai, Tianle and Lewis, Patrick and Chen, Deming},
  journal={Advances in Neural Information Processing Systems},
  volume={37},
  pages={22947--22970},
  year={2024}
}

@inproceedings{
zhu2025tactic,
title={Tactic: Adaptive Sparse Attention with Clustering and Distribution Fitting for Long-Context {LLM}s},
author={Kan Zhu and Tian Tang and Qinyu Xu and Zhan Jin and Yile Gu and Zhichen Zeng and Rohan Kadekodi and Liangyu Zhao and Ang Li and Arvind Krishnamurthy and Baris Kasikci},
booktitle={The Fourteenth International Conference on Learning Representations},
year={2026},
url={https://openreview.net/forum?id=tJod11fK1A}
}

@article{meta2025llama,
  title={The llama 4 herd: The beginning of a new era of natively multimodal ai innovation},
  author={Meta, AI},
  journal={https://ai. meta. com/blog/llama-4-multimodal-intelligence/, checked on},
  volume={4},
  number={7},
  pages={2025},
  year={2025}
}

@article{team2025kimi,
  title={Kimi linear: An expressive, efficient attention architecture},
  author={The Kimi Team},
  journal={arXiv preprint arXiv:2510.26692},
  year={2025}
}

@misc{dao2023flashdecoding,
  title        = {Flash-Decoding for long-context inference},
  author       = {Dao, Tri and Haziza, Daniel and Massa, Francisco and Sizov, Grigory},
  year         = {2023},
  month        = {10},
  day          = {12},
  howpublished = {\emph{CRFM Blog}},
  url          = {https://crfm.stanford.edu/2023/10/12/flashdecoding.html},
  note         = {Accessed: 2025-12-01}
}

@inproceedings{
gonccalves2026adasplash,
title={AdaSplash-2: Faster Differentiable Sparse Attention},
author={Nuno Gon{\c{c}}alves and Hugo Pitorro and Vlad Niculae and Edoardo Ponti and Lei Li and Andre Martins and Marcos Vinicius Treviso},
booktitle={Forty-third International Conference on Machine Learning},
year={2026},
url={https://openreview.net/forum?id=7qpvff2gWI}
}

@misc{hsieh2024rulerwhatsrealcontext,
      title={RULER: What's the Real Context Size of Your Long-Context Language Models?}, 
      author={Cheng-Ping Hsieh and Simeng Sun and Samuel Kriman and Shantanu Acharya and Dima Rekesh and Fei Jia and Yang Zhang and Boris Ginsburg},
      year={2024},
      eprint={2404.06654},
      archivePrefix={arXiv},
      primaryClass={cs.CL},
      url={https://arxiv.org/abs/2404.06654}, 
}
\bibliographystyle{iclr2027_conference}

\newpage
\appendix

\begin{figure}[t]
    \centering
    \includegraphics[width=1\linewidth]{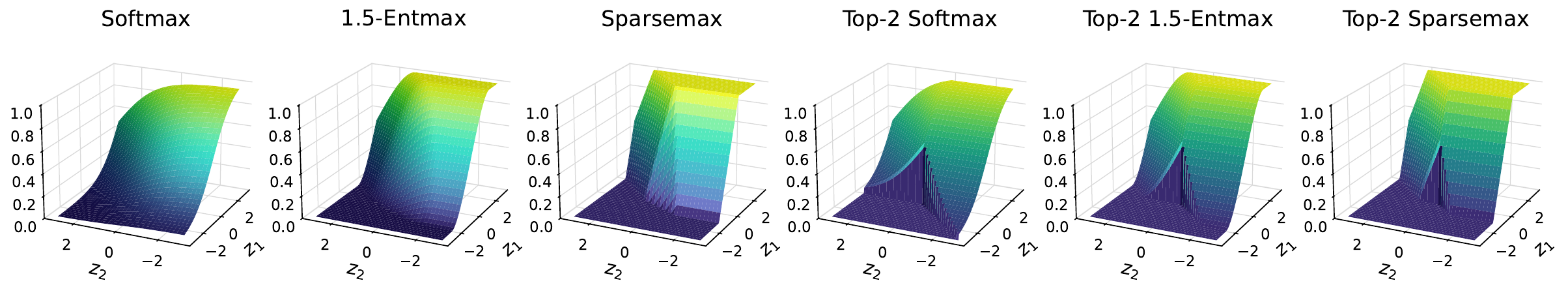}
    \caption{Visualization of softmax and $\alpha$-entmax for different values of $\alpha$, alongside top-$k$ variants with $k=2$.
    Each panel shows how $p_0$ varies for the input $\bm{z} = [0, z_1, z_2]$.
    }
    \label{fig:entmax_visualizations_app}
\end{figure}

\section{The \texorpdfstring{$\alpha$}{alpha}-entmax Transformation}
\label{app:entmax}

For $\alpha>1$, the $\alpha$-entmax transformation~\citep{peters-etal-2019-sparse} maps scores $\bm{s}\in\mathbb{R}^n$ to a probability distribution $\bm{p}\in\triangle_n$ of the form
\begin{align}
    p_i
    &=
    \left[
        (\alpha-1)s_i-\tau
    \right]_+^{1/(\alpha-1)},
    \label{eq:app-entmax-def}
\end{align}
where $\tau\in\mathbb{R}$ is chosen so that
\begin{align}
    \sum_{i=1}^{n}
    p_i
    &=
    1.
\end{align}
The notation $[x]_+=\max\{x,0\}$ is applied element-wise.
The support of $\bm{p}$ is
\begin{align}
    \simark
    &=
    \supp(\bm{p})
    =
    \left\{
        i:
        p_i>0
    \right\}
    \\
    &=
    \left\{
        i:
        (\alpha-1)s_i>\tau
    \right\}. \nonumber
    \label{eq:app-entmax-support}
\end{align}
This threshold characterization is central to our method since a token contributes to entmax attention if and only if its scaled score is above the threshold.
Special cases include sparsemax at $\alpha=2$~\citep{martins2016softmax} and softmax in the limit $\alpha\to1$~\citep{peters-etal-2019-sparse,blondel-entmax}. For $\alpha>1$, the transformation can produce exact zeros, unlike softmax. $\alpha$ controls the propensity to sparsity, with larger values of $\alpha$ generally encouraging sparser distributions.
Figure~\ref{fig:entmax_visualizations_app} illustrates $\alpha$-entmax for $\alpha \in \{1.0, 1.5, 2.0\}$ and top-$k$ variants ($k=2$) for a 3D input.

\section{Additional Details for \methodname}
\label{app:method-proofs}

This appendix provides additional details for the page scoring and Gaussian selection rules used in Section~\ref{sec:method}.

\subsection{Deterministic Page Score Bound}
\label{app:page-bound-proof}

Recall that page $p$ stores coordinate-wise key minima and maxima:
\begin{align}
    k_{\min,i}^{(p)}
    =
    \min_{j\in\mathcal{P}_p}
    k_{j,i},
    \qquad
    k_{\max,i}^{(p)}
    =
    \max_{j\in\mathcal{P}_p}
    k_{j,i}.
\end{align}
For any token $j\in\mathcal{P}_p$,
\begin{align}
    k_{j,i}
    &\in
    \left[
        k_{\min,i}^{(p)},
        k_{\max,i}^{(p)}
    \right].
\end{align}

\begin{proposition}[Page score upper bound]
\label{prop:page-bound}
For any query $\bm{q}$ and token $j\in\mathcal{P}_p$,
\begin{align}
    s_j(\bm{q})
    &=
    \frac{\bm{q}^{\top}\bm{k}_j}{\sqrt d}
    \le
    \bar{s}_{\mathrm{box}}^{(p)}(\bm{q}),
\end{align}
where
\begin{align}
    \bar{s}_{\mathrm{box}}^{(p)}(\bm{q})
    &=
    \frac{1}{\sqrt d}
    \sum_{i=1}^{d}
    \max
    \left(
        q_i k_{\min,i}^{(p)},
        q_i k_{\max,i}^{(p)}
    \right).
\end{align}
\end{proposition}

\begin{proof}
For each coordinate $i$, the scalar $k_{j,i}$ lies between the coordinate-wise minimum and maximum of the page. Therefore,
\begin{align}
    q_i k_{j,i}
    &\le
    \max
    \left(
        q_i k_{\min,i}^{(p)},
        q_i k_{\max,i}^{(p)}
    \right).
\end{align}
If $q_i\ge0$, the maximum is attained at $k_{\max,i}^{(p)}$; if $q_i<0$, it is attained at $k_{\min,i}^{(p)}$. Summing over coordinates gives
\begin{align}
    \bm{q}^{\top}\bm{k}_j
    &=
    \sum_{i=1}^{d}
    q_i k_{j,i}
    \\
    &\le
    \sum_{i=1}^{d}
    \max
    \left(
        q_i k_{\min,i}^{(p)},
        q_i k_{\max,i}^{(p)}
    \right). \nonumber
\end{align}
Dividing by $\sqrt d$ proves the result.
\end{proof}

Importantly, this bound is conservative because it maximizes each coordinate independently over an axis-aligned box. While the maximizing vector may not correspond to an actual key, the bound is cheap and safe to compute.

\subsection{Support-Superset Selection from Page Bounds}
\label{app:support-superset}

\begin{proposition}[No false negatives from deterministic page bounds]
\label{prop:page-support-superset}
Let $\bm{p}=\entmax(\bm{s})$ with threshold $\tau$, and suppose $\hat{\tau}\le\tau$. Define
\begin{align}
    \mathcal{C}_{\mathrm{page}}
    &=
    \left\{
        p:
        (\alpha-1)
        \bar{s}_{\mathrm{box}}^{(p)}
        >
        \hat{\tau}
    \right\}.
\end{align}
Then every page containing an entmax support token is selected:
\begin{align}
    \supp(\bm{p})
    &\subseteq
    \bigcup_{p\in\mathcal{C}_{\mathrm{page}}}
    \mathcal{P}_p.
\end{align}
\end{proposition}

\begin{proof}
Let $i\in\supp(\bm{p})$, and let $p(i)$ be the page containing token $i$. Since $i$ is in the entmax support,
\begin{align}
    (\alpha-1)s_i
    &>
    \tau.
\end{align}
Because $\hat{\tau}\le\tau$,
\begin{align}
    (\alpha-1)s_i
    &>
    \hat{\tau}.
\end{align}
By Proposition~\ref{prop:page-bound},
\begin{align}
    s_i
    &\le
    \bar{s}_{\mathrm{box}}^{(p(i))}.
\end{align}
Therefore,
\begin{align}
    (\alpha-1)
    \bar{s}_{\mathrm{box}}^{(p(i))}
    &\ge
    (\alpha-1)s_i
    \\
    &>
    \hat{\tau}. \nonumber
\end{align}
Thus page $p(i)$ is selected. Since this holds for every support token, all support tokens lie in selected pages.
\end{proof}

This proposition gives an ideal deterministic picture: if page maxima are upper-bounded and the threshold estimate is conservative, page selection cannot miss support tokens. The Gaussian selector trades this worst-case guarantee for adaptivity.

\section{Approximation Error Under Truncation}
\label{app:approx-truncation}

This appendix proves the truncation results used in Section~\ref{sec:theory}.

\subsection{Proof of Proposition~\ref{prop:truncation-bound}}

\begin{proof}
Let $\delta > 0$, and recall that $\delta=\sum_{i\in\Idrop}p_i$ and
\begin{align}
    \tilde{p}_i
    &=
    \begin{cases}
        p_i/(1-\delta), & i\in\Ikeep,\\
        0, & i\in\Idrop.
    \end{cases}
\end{align}
We first compute the $L^1$ distance between $\bm{p}$ and $\tilde{\bm{p}}$. On dropped indices,
\begin{align}
    \sum_{i\in\Idrop}
    |p_i-\tilde{p}_i|
    =
    \sum_{i\in\Idrop}
    p_i
    =
    \delta.
\end{align}
On kept indices,
\begin{align}
    \sum_{i\in\Ikeep}
    |\tilde{p}_i-p_i|
    &=
    \sum_{i\in\Ikeep}
    \left(
        \frac{p_i}{1-\delta}
        -
        p_i
    \right)
    \\ \nonumber
    &=
    \left(
        \frac{1}{1-\delta}
        -
        1
    \right)
    \sum_{i\in\Ikeep}
    p_i
    \\ \nonumber
    &=
    \frac{\delta}{1-\delta}
    (1-\delta)
    \\ \nonumber
    &=
    \delta.
\end{align}
Therefore,
\begin{align}
    \|\bm{p}-\tilde{\bm{p}}\|_1
    &=
    2\delta.
\end{align}

Now,
\begin{align}
    \|\bm{o}-\tilde{\bm{o}}\|_2
    &=
    \left\|
        \sum_{i=1}^{n}
        (p_i-\tilde{p}_i)\bm{v}_i
    \right\|_2
    \\ \nonumber
    &\le
    \sum_{i=1}^{n}
    |p_i-\tilde{p}_i|
    \|\bm{v}_i\|_2
    \\ \nonumber
    &\le
    B
    \|\bm{p}-\tilde{\bm{p}}\|_1
    \\ \nonumber
    &=
    2B\delta.
\end{align}

To see tightness, consider two tokens with probabilities $p_1=1-\delta$ and $p_2=\delta$, and keep only token the first token. 
Let $\bm{v}_1=B\bm{u}$ and $\bm{v}_2=-B\bm{u}$ for a unit vector $\bm{u}$. Then
\begin{align}
    \bm{o}
    =
    (1-\delta)B\bm{u}
    -
    \delta B\bm{u}
    =
    (1-2\delta)B\bm{u},
    \qquad
    \tilde{\bm{o}}
    =
    B\bm{u}.
\end{align}
Thus
\begin{align}
    \|\bm{o}-\tilde{\bm{o}}\|_2
    &=
    2B\delta.
\end{align}
\end{proof}

\subsection{Proof of Proposition~\ref{prop:exact-sparse}}
\label{app:exact-sparse}

\begin{proof}
If $\simark\subseteq\Ikeep$, then every dropped index has zero entmax probability:
\begin{align}
    i\in\Idrop
    &\implies
    i\notin\simark
    \\ \nonumber
    &\implies
    p_i=0.
\end{align}
Therefore,
\begin{align}
    \delta
    &=
    \sum_{i\in\Idrop}
    p_i
    =
    0.
\end{align}
Since $\delta=0$, the renormalized distribution satisfies $\tilde{\bm{p}}=\bm{p}$. Hence
\begin{align}
    \tilde{\bm{o}}
    &=
    \sum_{i=1}^{n}
    \tilde{p}_i\bm{v}_i
    =
    \sum_{i=1}^{n}
    p_i\bm{v}_i
    =
    \bm{o}.
\end{align}
\end{proof}

\section{Closed-Form Distributional Entmax Thresholds for Special \texorpdfstring{$\alpha$}{alpha}}
\label{app:detg-closed-form}

Recall from Section~\ref{sec:candidate-selection} that the Gaussian-aware entmax selector estimates the entmax threshold by replacing the empirical sum over token scores with a page-wise Gaussian approximation. The exact entmax normalization condition is
\begin{align}
    \sum_{i=1}^{n}
    \left[
        (\alpha-1)s_i-\tau
    \right]_+^\beta
    =
    1,
    \quad \text{with} \quad 
    \beta
    =
    \frac{1}{\alpha-1}.
\end{align}
For each page $p$, we model token scores by a scalar random variable
\begin{align}
    S^{(p)}
    &\sim
    \mathcal{N}
    \left(
        \mu_p,
        \sigma_p^2
    \right).
\end{align}
Under this model, the normalization condition is approximated by
\begin{align}
    \sum_{p=1}^{M}
    |\mathcal{P}_p|
    \,
    \E_{S^{(p)}}
    \left[
        g_\alpha(S^{(p)};\tau)
    \right]
    &=
    1,
    \label{eq:app-det-page-equation}
\end{align}
where
\begin{align}
    g_\alpha(s;\tau)
    &=
    \left[
        (\alpha-1)s-\tau
    \right]_+^\beta.
\end{align}
Thus, each page contributes its expected unnormalized entmax mass, multiplied by the number of tokens in the page.
Let $a=\alpha-1$. For a generic Gaussian score $S\sim\mathcal{N}(\mu,\sigma^2)$, define
\begin{align}
    Y
    &=
    aS-\tau.
\end{align}
Then
\begin{align}
    Y
    \sim
    \mathcal{N}
    \left(
        \mu_Y,
        \sigma_Y^2
    \right),
    \quad
    \mu_Y
    =
    a\mu-\tau,
    \quad
    \sigma_Y
    =
    a\sigma.
\end{align}
Since
\begin{align}
    g_\alpha(S;\tau)
    &=
    (Y)_+^\beta,
\end{align}
we need the truncated Gaussian moment
\begin{align}
    \E
    \left[
        g_\alpha(S;\tau)
    \right]
    &=
    \E
    \left[
        (Y)_+^\beta
    \right]
    =
    \E
    \left[
        Y^\beta
        \mathbf{1}\{Y>0\}
    \right].
\end{align}
Let
\begin{align}
    t
    &=
    \frac{\mu_Y}{\sigma_Y},
\end{align}
and let $\phi$ and $\Phi$ denote the standard normal pdf and cdf. When $\beta\in\mathbb{N}$, this expectation is a truncated Gaussian moment with a closed form. We use the cases $\beta=1,2,3$, corresponding to $\alpha=2,3/2,4/3$.\footnote{These values are convenient because $\beta=1/(\alpha-1)$ is an integer. For non-integer $\beta$, the same distributional threshold equation remains valid, but the expectation generally requires numerical evaluation.}

\paragraph{Case $\alpha=2$ ($\beta=1$).}
Here $a=1$ and
\begin{align}
    g_2(s;\tau)
    &=
    (s-\tau)_+.
\end{align}
The first truncated Gaussian moment is
\begin{align}
    \E
    \left[
        Y\mathbf{1}\{Y>0\}
    \right]
    &=
    \mu_Y\Phi(t)
    +
    \sigma_Y\phi(t).
\end{align}
Therefore,
\begin{align}
    \E
    \left[
        g_2(S;\tau)
    \right]
    &=
    (\mu-\tau)
    \Phi
    \left(
        \frac{\mu-\tau}{\sigma}
    \right)
    +
    \sigma
    \phi
    \left(
        \frac{\mu-\tau}{\sigma}
    \right).
\end{align}

\paragraph{Case $\alpha=3/2$ ($\beta=2$).}
Here $a=1/2$ and
\begin{align}
    g_{3/2}(s;\tau)
    &=
    \left(
        \frac{1}{2}s-\tau
    \right)_+^2.
\end{align}
The second truncated Gaussian moment is
\begin{align}
    \E
    \left[
        Y^2\mathbf{1}\{Y>0\}
    \right]
    &=
    \left(
        \mu_Y^2+\sigma_Y^2
    \right)
    \Phi(t)
    +
    \mu_Y\sigma_Y\phi(t).
\end{align}
Thus,
\begin{align}
    \E
    \left[
        g_{3/2}(S;\tau)
    \right]
    &=
    \left(
        \mu_Y^2+\sigma_Y^2
    \right)
    \Phi
    \left(
        \frac{\mu_Y}{\sigma_Y}
    \right)
    +
    \mu_Y\sigma_Y
    \phi
    \left(
        \frac{\mu_Y}{\sigma_Y}
    \right),
    \\
    \text{with } \mu_Y
    &=
    \frac{1}{2}\mu-\tau,
    \text{ and }
    \sigma_Y
    =
    \frac{1}{2}\sigma.
\end{align}

\paragraph{Case $\alpha=4/3$ ($\beta=3$).}
Here $a=1/3$ and
\begin{align}
    g_{4/3}(s;\tau)
    &=
    \left(
        \frac{1}{3}s-\tau
    \right)_+^3.
\end{align}
The third truncated Gaussian moment is
\begin{align}
    \E
    \left[
        Y^3\mathbf{1}\{Y>0\}
    \right]
    &=
    \left(
        \mu_Y^3
        +
        3\mu_Y\sigma_Y^2
    \right)
    \Phi(t)
    +
    \left(
        \mu_Y^2\sigma_Y
        +
        2\sigma_Y^3
    \right)
    \phi(t).
    \label{eq:app-third-trunc-moment}
\end{align}
Therefore,
\begin{align}
    \E
    \left[
        g_{4/3}(S;\tau)
    \right]
    &=
    \left(
        \mu_Y^3
        +
        3\mu_Y\sigma_Y^2
    \right)
    \Phi
    \left(
        \frac{\mu_Y}{\sigma_Y}
    \right)
    +
    \left(
        \mu_Y^2\sigma_Y
        +
        2\sigma_Y^3
    \right)
    \phi
    \left(
        \frac{\mu_Y}{\sigma_Y}
    \right),
    \\
    \text{with } \mu_Y
    &=
    \frac{1}{3}\mu-\tau,
    \text{ and }
    \sigma_Y
    =
    \frac{1}{3}\sigma.
\end{align}

\paragraph{Solving for the distributional threshold.}
For each page $p$, we evaluate the appropriate closed form using $\mu=\mu_p$ and $\sigma=\sigma_p$, then substitute the result into Eq.~\eqref{eq:app-det-page-equation}. The estimated threshold $\hat{\tau}$ is the solution of the scalar equation
\begin{align}
    \sum_{p=1}^{M}
    |\mathcal{P}_p|
    \,
    \E_{S^{(p)}}
    \left[
        g_\alpha(S^{(p)};\hat{\tau})
    \right]
    &=
    1.
\end{align}
We solve this equation with a small fixed number of Newton or Halley iterations over $\tau$. Each iteration has constant cost per page because each page contributes only through $\mu_p$, $\sigma_p$, and evaluations of $\phi$ and $\Phi$. Importantly, this step does not require loading the full KV cache.
For generic $\alpha>1$ with non-integer $\beta$, the same distributional equation can still be used, but the expectation $\E[g_\alpha(S;\tau)]$ no longer reduces to a low-order polynomial truncated Gaussian moment. In that case, the expectation can be evaluated numerically.

\section{Experiments}
\label{app:additional_experiments}

\subsection{Setup}
\label{app:experimental_setup}

\paragraph{Hardware and models.}
Experiments run on NVIDIA H200 GPUs.
Model-based evaluations use the 1B-parameter softmax and
entmax NAPE checkpoints from
AdaSplash-2~\citep{gonccalves2026adasplash}, trained with
32k contexts and ALiBi bias. The checkpoints use grouped-query
attention with 24 query heads and 12 KV heads; the entmax
checkpoint uses $\alpha=1.5$.

\paragraph{Decoding efficiency.}
We benchmark the attention computation using 24 heads, with a head dimension of 32 and a batch size of 8.
Timing includes page scoring, selection, selected KV reads,
and attention computation, but excludes prefill and other
model components.
We compare full-cache softmax FlashDecoding~\citep{dao2023flashdecoding},
full-cache entmax, and sparse variants.
The full-cache entmax reference materializes the score matrix
and applies AdaSplash-2's \texttt{triton\_entmax\_v2} kernel.
Softmax and entmax top-$k$ share our Triton implementation
of Quest's page-scoring and selection
pipeline~\citep{tang2024quest}.
After 50 warm-up iterations, we report median wall-clock latency over
50 repeated iterations, synchronized with
\texttt{torch.cuda.synchronize()} before and after each measurement.

\paragraph{Approximation quality.}
We measure dropped probability mass $\delta$, support-token
recall $\rho=|\simark\cap\Ikeep|/|\simark|$, and relative output
error $R=\|\bm{o}-\tilde{\bm{o}}\|_2/\|\bm{o}\|_2$ against
the corresponding full-cache attention output.
Here $\rho$ measures support cardinality, not probability mass.
Coverage is the fraction of KV entries selected for attention.
Fixed-budget comparisons use matched coverage; Gaussian-aware
selection determines its own coverage.
The synthetic top-$k$ diagnostics use standard-normal Q/K/V.
Section~\ref{subsubsec:gaussian_entmax_results_approx}
evaluates Gaussian-aware selection on captured model activations.

\paragraph{Downstream benchmarks.}
We evaluate passkey retrieval across context lengths and
passkey positions, including randomized positions, and report
the percentage of correct answers.
RULER~\citep{hsieh2024rulerwhatsrealcontext} evaluates retrieval,
aggregation, tracking, and question answering at 4k--32k contexts.
Quest, (Softmax Top-k), Entmax- Top-k, Twilight and SnapKV use a 1024-token budget; Gaussian-aware
selection uses a variable budget.
PG19~\citep{rae2019compressive} perplexity is evaluated on
test sequences up to 32k tokens under full-cache and sparse
decoding, with fixed budgets of 1024, 2048, 4096, and 8192 tokens.

\subsection{Quality metrics using real activations}
\label{app:real_activations_quality}

We evaluate our methods using Q/K/V activations
captured from the entmax and softmax checkpoints on six RULER tasks:
\texttt{niah\_multikey\_1}, \texttt{niah\_multikey\_2},
\texttt{niah\_multivalue}, \texttt{niah\_single\_1},
\texttt{niah\_single\_2}, and \texttt{vt} and report quality metrics, the Gaussian Aware uses an adaptive budget while the Top-k variants use matched budgets.

\subsubsection{Gaussian Entmax}
\label{subsubsec:gaussian_entmax_results_approx}

Table~\ref{tab:gaussian_quality} reports means and standard
deviations over capture/task/layer/step records, comparing
sparse outputs with full-cache entmax on the same activations.
The selector determines its page budget based on its $\tau$ estimate, so coverage reports the selected fraction.

\begin{table}[t]
\centering
\small
\setlength{\tabcolsep}{4pt}
\caption{Gaussian-aware selection on real model activations.
Coverage, dropped mass $\delta$, support-token recall $\rho$,
and relative output error $R$ are reported as mean $\pm$
standard deviation.}
\label{tab:gaussian_quality}
\begin{tabular}{lcccc}
\toprule
KV length & Coverage & $\delta\downarrow$ &
$\rho\uparrow$ & $R\downarrow$ \\
\midrule
16k
& $0.365 \pm 0.106$
& $0.0003 \pm 0.0011$
& $0.9991 \pm 0.0021$
& $0.0019 \pm 0.0068$ \\
32k
& $0.269 \pm 0.057$
& $0.0021 \pm 0.0039$
& $0.9984 \pm 0.0022$
& $0.0093 \pm 0.0146$ \\
\bottomrule
\end{tabular}
\end{table}

Across these records, the selector retains an average of
$99.88\%$ of support tokens at $32.2\%$ coverage, with
$0.52\%$ relative output error.
Mean coverage decreases from $36.5\%$ at 16k to $26.9\%$
at 32k, while mean relative error increases from $0.19\%$
to $0.93\%$.
The minimum observed support-token recall is $98.48\%$,
and the maximum relative output error is $6.35\%$.
These results demonstrate high empirical support recall at comparatively low coverages.

\subsubsection{Entmax top-$k$ and softmax top-$k$ (Quest) under real activations}
\label{subsubsec:Real activations}

Table ~\ref{tab:real-activations-topk} shows quality metrics such as support recall, percentage of probability mass dropped and relative error collected from real activations and compared with a full kv cache execution both for softmax and $\entmax$. We can see that across all coverage levels $\entmax$ drops less probability mass, retains a larger percentage of the support and shows a lower relative error.

\begin{table}[t]
\small
\centering
\caption{Quality metrics for  entmax and softmax top-$k$ attention on real activations.
Values are mean $\pm$ sample standard deviation. $R$ denotes relative output
error. Results are aggregated over the tasks available for both models.}
\label{tab:real-activations-topk}
\begin{tabular}{llcccc}
\toprule
KV length & Method & Coverage &
$\delta\downarrow$ & $\rho\uparrow$ & $R\downarrow$ \\
\midrule
16k & Entmax top-k & 10\%
& $0.1737 \pm 0.1868$
& $0.7684 \pm 0.1784$
& $0.0916 \pm 0.0861$ \\

16k & Softmax top-k & 10\%
& $0.2428 \pm 0.1676$
& $0.1 \pm 0.0000$
& $0.1738 \pm 0.0819$ \\
\addlinespace

16k & Entmax top-k & 12.5\%
& $0.1646 \pm 0.1820$
& $0.7823 \pm 0.1752$
& $0.0824 \pm 0.0805$ \\

16k & Softmax top-k & 12.5\%
& $0.2315 \pm 0.1641$
& $0.125 \pm 0.0000$
& $0.1590 \pm 0.0756$ \\
\addlinespace

16k & Entmax top-k & 25\%
& $0.1289 \pm 0.1563$
& $0.8354 \pm 0.1556$
& $0.0549 \pm 0.0618$ \\

16k & Softmax top-k & 25\%
& $0.1854 \pm 0.1442$
& $0.25 \pm 0.0000$
& $0.1111 \pm 0.0560$ \\
\addlinespace

16k & Entmax top-k & 50\%
& $0.0760 \pm 0.1004$
& $0.9056 \pm 0.1071$
& $0.0273 \pm 0.0365$ \\

16k & Softmax top-k & 50\%
& $0.1128 \pm 0.0984$
& $0.5 \pm 0.0000$
& $0.0593 \pm 0.0353$ \\
\midrule

32k & Entmax top-k & 10\%
& $0.1756 \pm 0.1812$
& $0.7687 \pm 0.1739$
& $0.0925 \pm 0.0980$ \\

32k & Softmax top-k & 10\%
& $0.2350 \pm 0.1689$
& $0.1 \pm 0.0000$
& $0.1582 \pm 0.0692$ \\
\addlinespace

32k & Entmax top-k & 12.5\%
& $0.1643 \pm 0.1772$
& $0.7857 \pm 0.1723$
& $0.0823 \pm 0.0913$ \\

32k & Softmax top-k & 12.5\%
& $0.2230 \pm 0.1661$
& $0.125 \pm 0.0000$
& $0.1432 \pm 0.0662$ \\
\addlinespace

32k & Entmax top-k & 25\%
& $0.1241 \pm 0.1531$
& $0.8445 \pm 0.1570$
& $0.0522 \pm 0.0684$ \\

32k & Softmax top-k & 25\%
& $0.1744 \pm 0.1477$
& $0.25 \pm 0.0000$
& $0.0991 \pm 0.0571$ \\
\addlinespace

32k & Entmax top-k & 50\%
& $0.0718 \pm 0.0969$
& $0.9115 \pm 0.1094$
& $0.0253 \pm 0.0396$ \\

32k & Softmax top-k & 50\%
& $0.1042 \pm 0.1016$
& $0.5 \pm 0.0000$
& $0.0519 \pm 0.0356$ \\
\bottomrule
\end{tabular}
\end{table}

\section{Implementation details}
\label{app:implementation-details}

Our kernel algorithm to compute paged candidates in entmax top-k attention is shown in Alg. \ref{alg:top-k-candidate-kernel}. The gaussian variant uses a small modification to the $\tau$ root-finding algorithm of AdaSplash~\citep{goncalves2025adasplash}; we present its kernel in Alg. \ref{alg:gaussian-decode}.

\begin{algorithm}[tb]
\caption{Paged candidate attention for Entmax Top-$k$}
\label{alg:top-k-candidate-kernel}
\small
\textbf{Input:}
$q_h,K,V,T,L,a_h$; selected logical pages $S$;
candidate-token count $N$; $\alpha\in\{3/2,2\}$;
refinement count $R\geq1$.

\textbf{Output:} Attention output $o_h$.
\begin{enumerate}
\item Partition the selected page slots into disjoint sets
$\mathcal{S}_1,\ldots,\mathcal{S}_P$.

\item \textbf{Materialize scores.}
In parallel over partitions, visit each slot $s$ and token offset
$t\in\{0,\ldots,B-1\}$.
Set $\ell=S_s$ and $j=\ell B+t$.
For $\ell\geq0$ and $j<L$, resolve $b=T_\ell$ and store
\[
Z_{s,t}\leftarrow(\alpha-1)
\left(
\frac{q_h^\top K_{b,t,g(h)}}{\sqrt d}
+a_h(j-L+1)
\right).
\]
Store $Z_{s,t}\leftarrow-\infty$ for invalid tokens.
Reduce partition maxima to obtain
\[
m\leftarrow\max_{s,t}Z_{s,t}.
\]
This is the only pass over the selected candidate set.

\item \textbf{Initialize the threshold.}
Build and reduce score histograms over $[m-1,m]$.
Use the histogram estimate to initialize $\tau$, clipped to
\[
\tau_{\rm lo}=m-1,
\qquad
\tau_{\rm hi}=m-N^{\,1-\alpha}.
\]

\item \textbf{Refine the threshold.}
For $r=1,\ldots,R$:
\begin{enumerate}
\item In parallel over partitions, read cached scores and form
\[
x_{s,t}\leftarrow[Z_{s,t}-\tau]_+.
\]
On the first iteration, store the page mask
\[
C_s\leftarrow
\left(\max_t Z_{s,t}>m-1\right).
\]

\item Reduce the normalization residual and its derivatives:
\[
(f,f',f'')\leftarrow
\left\{
\begin{array}{ll}
\left(
\sum x_{s,t}^{2}-1,\;
-2\sum x_{s,t},\;
2\,|\{(s,t):Z_{s,t}>\tau\}|
\right),
& \alpha=3/2,\\[2mm]
\left(
\sum x_{s,t}-1,\;
-|\{(s,t):Z_{s,t}>\tau\}|,\;
0
\right),
& \alpha=2.
\end{array}
\right.
\]

\item If $f>0$, set $\tau_{\rm lo}\leftarrow\tau$.
If $f<0$, set $\tau_{\rm hi}\leftarrow\tau$.

\item Form the Halley proposal
\[
\tau_{\rm H}\leftarrow
\tau-\frac{2ff'}{2(f')^2-ff''}.
\]
Accept it if it lies within the updated bracket, allowing
numerical tolerance $\varepsilon=10^{-6}$;
otherwise set
$\tau\leftarrow(\tau_{\rm lo}+\tau_{\rm hi})/2$.
\end{enumerate}

\item \textbf{Accumulate values.}
In parallel over partitions, initialize $o_p\leftarrow0$.
For each slot $s\in\mathcal{S}_p$ with $C_s$ true,
read its cached scores and identify tokens with $Z_{s,t}>\tau$.
For these tokens only, resolve $b=T_{S_s}$, load the value, and add
\[
o_p\leftarrow o_p+
[Z_{s,t}-\tau]_+^{\,1/(\alpha-1)}
V_{b,t,g(h)}.
\]

\item \textbf{Return}
$\displaystyle o_h\leftarrow\sum_{p=1}^{P}o_p$.
\end{enumerate}
\end{algorithm}

\begin{algorithm}[tb]
\caption{Gaussian decode}
\label{alg:gaussian-decode}
\small
\textbf{Input:}
Query $q_h$; paged caches $K,V$; block table $T$;
page key statistics $\mathcal{M}$; sequence length $L$;
ALiBi slope $a_h$; selection parameters $\theta$;
number of $\tau$ refinement iterations $R\geq 0$.

\textbf{Output:} Attention output $o_h$.
\begin{enumerate}
\item Compute page and global score statistics:
\[
(\mu_{\rm page},\sigma_{\rm page},\mu,\sigma,N_{\rm hist},P_{\rm hist})
\leftarrow
\mathrm{GaussianScoreStats}(q_h,\mathcal{M},T,L,a_h).
\]

\item Estimate the threshold $\widehat{\tau}$ as described in \ref{app:detg-closed-form}:

\item Select and pack logical candidate pages estimated to have tokens above the support, and the last page:
\[
(S,N)\leftarrow
\mathrm{SelectAndPackPages}
(\mu_{\rm page},\sigma_{\rm page},
 \widehat{\tau},\sigma,L,B,\theta).
\]
Here $N$ counts valid candidate tokens.

\item Materialize candidate scores and partial maxima:
\[
(Z,m_{\rm partial})\leftarrow
\mathrm{ScoreSelectedPages}(q_h,K,T,S,N,L,a_h).
\]
The kernel partitions the compact candidate-token range across
parallel splits and stores scores in FP32.

\item \textbf{Initialize the threshold bracket.}
If $R>0$, reduce the partial maxima to obtain
$m=\max_j z_j$ over the selected tokens.
For $\alpha=3/2$, the exact candidate-set threshold satisfies
\[
\tau^\star\in[m-1,m].
\]
This bracket contains the root of
$f(\tau)=\sum_j[z_j-\tau]_+^2-1$:
the lower endpoint gives mass at least one, while the upper
endpoint gives zero mass.

Call the bracket-initialization kernel to set
\[
\tau_{\rm lo}\leftarrow m-1,
\qquad
\tau_{\rm hi}\leftarrow m,
\qquad
\tau\leftarrow
\min\{\tau_{\rm hi},\max\{\tau_{\rm lo},\widehat{\tau}\}\}.
\]
The safeguard state $\mathcal{B}$ stores these bounds and the
previous iterate and residual, initialized to $\tau$ and $0$.
During correction, a positive residual raises the lower bound;
a negative residual lowers the upper bound.
If $R=0$, use $\tau\leftarrow\widehat{\tau}$ without bracketing.

\item For $r=1,\ldots,R$:
\begin{enumerate}
\item
$\mathcal{A}\leftarrow
\mathrm{AccumulateThresholdStatistics}(Z,N,\tau)$.
\item
$(\tau,\mathcal{B})\leftarrow
\mathrm{SafeguardedThresholdUpdate}
(\mathcal{A},\tau,\mathcal{B},r)$.
\end{enumerate}

\item \textbf{Accumulate selected values.}
Launch $\mathrm{SelectedValueAccumulation}$ over the same
candidate-token splits used for scoring.
Each split $p$ initializes a vector accumulator $o_p\leftarrow0$
and a scalar mass $M_p\leftarrow0$, then processes its cached
scores in tiles.

For each valid compact token index $i$, compute
\[
w_i\leftarrow[Z_i-\tau]_+^2,
\qquad
M_p\leftarrow M_p+w_i.
\]
If every weight in a tile is zero, skip its value-cache loads.
Otherwise, map each valid compact index to its selected page
and offset:
\[
s=\lfloor i/B\rfloor,
\qquad
t=i\bmod B,
\qquad
\ell=S_s,
\qquad
b=T_\ell.
\]
Load the corresponding values and accumulate
\[
o_p\leftarrow
o_p+\sum_{i\ \mathrm{in\ the\ tile}}
w_i V_{T_{S_{\lfloor i/B\rfloor}},\,i\bmod B,\,g(h)}.
\]
Only indices $i<N$ participate, excluding padding and invalid
tail tokens. Store $o_p$ and $M_p$ in
$O_{\rm partial}$ and $M_{\rm partial}$ for the final reduction.

\item
$(o_h,M)\leftarrow
\mathrm{ReduceOutput}(O_{\rm partial},M_{\rm partial})$.

\item \textbf{Return} $o_h$.
\end{enumerate}
\end{algorithm}

\end{document}